\documentclass[twoside,11pt]{article}

\usepackage{jmlr2e}

\jmlrheading{1}{2021}{1-48}{4/00}{10/00}{meila00a}{Rodrigue Siry, Ryan Webster, Loic Simon and Julien Rabin}

\ShortHeadings{Theoretical Equivalence of Several Trade-Off Curves Assessing Statistical Proximity}{Siry and Webster and Simon and Rabin}
\firstpageno{1}

\usepackage{graphicx}
\usepackage{lipsum}
\usepackage{subfigure}
\usepackage{booktabs} 
\usepackage{tikz}
\usetikzlibrary{shadows}
\usetikzlibrary{patterns}

\usepackage{mathtools}
\usepackage{amssymb}
\usepackage{xcolor}
\usepackage{natbib}
\defcitealias{van2010renyi}{van Erven (2010)}

\usepackage[]{algorithm2e}

\usepackage{hyperref}
\usepackage[utf8]{inputenc}
\usepackage{amsmath}
\usepackage{dsfont}
\usepackage{enumitem}
\makeatletter
\def\namedlabel#1#2{\begingroup
    #2%
    \def\@currentlabel{#2}%
    \phantomsection\label{#1}\endgroup
}
\makeatother

\iffalse
\usepackage[]{todonotes} 
\newcommand{\JR}[1]{\todo[color=blue!20,size=\small]{{\bf JR:} #1}}

\newcommand{\RS}[1]{\todo[color=red!20,size=\small]{{\bf RS:} #1}}
\newcommand{\RW}[1]{\todo[color=violet!20,size=\small]{{\bf RW:} #1}}

\else
\newcommand{\JR}[1]{}

\newcommand{\RS}[1]{}
\newcommand{\RW}[1]{}

\fi

\def \ifempty#1{\def\temp{#1} \ifx\temp\empty }

\newcommand{\N}{\mathbb{N}}
\newcommand{\R}{\mathbb{R}}
\newcommand{\1}{\mathds{1}}
\newcommand{\measSpace}{\Omega}
\newcommand{\measSet}{\mathcal M(\measSpace)}
\newcommand{\pmeasSet}{\mathcal M^+(\measSpace)}
\newcommand{\probaSet}{\mathcal M_p(\measSpace)}
\newcommand{\pp}[1]{~#1\text{-a.e.}}

\newcommand{\AC}{\text{AC}}
\newcommand{\PRD}{\text{PRD}}
\newcommand{\LD}{\text{LD}}
\newcommand{\dLD}{F_{LD}}

\newcommand{\risk}{R}
\newcommand{\tiA}{TI}
\newcommand{\tiB}{TI'}

\def\T{Q}
\def\SectionSymbol{§}
\def\S{P}
\def\dom{\mathrm{dom}}
\def\sign{\mathrm{sign}}
\def\range{\mathrm{range}}

\usepackage{xargs}
\newcommand{\LRS}{A_\lambda}

\DeclareMathOperator*{\argmin}{arg\,min}
\DeclareMathOperator*{\argmax}{arg\,max}
\DeclareMathOperator*{\esssup}{ess\,sup}

\iffalse
\newcommand{\revision}[1]{{\color{red} #1}}
\newcommand{\revisioncomment}[1]{{\color{blue} #1}}
\else
\newcommand{\revision}[1]{{ #1}}
\newcommand{\revisioncomment}[1]{}
\fi

\begin{document}

\title{On the Theoretical Equivalence of Several Trade-Off Curves \\Assessing Statistical Proximity}

\author{\name Rodrigue Siry \email rodrigue.siry@safrangroup.com\\
       \addr Safran Electronics and Defense
       \AND
       \name Ryan Webster \email ryan.webster@unicaen.fr\\
       \addr NORMANDIE UNIV, UNICAEN, ENSICAEN, CNRS, GREYC, 14000 CAEN, FRANCE
       \AND
       \name Loic Simon \email loic.simon@ensicaen.fr\\
       \addr NORMANDIE UNIV, UNICAEN, ENSICAEN, CNRS, GREYC, 14000 CAEN, FRANCE
       \AND
       \name Julien Rabin \email julien.rabin@unicaen.fr\\
       \addr NORMANDIE UNIV, UNICAEN, ENSICAEN, CNRS, GREYC, 14000 CAEN, FRANCE
       }

\editor{Kevin Murphy and Bernhard Sch{\"o}lkopf}

\maketitle

\begin{abstract}
The recent advent of powerful generative models has triggered the renewed development of quantitative measures to assess the proximity of two probability distributions.
As the scalar Frechet inception distance remains popular, several methods have explored computing entire curves, which reveal the trade-off between the fidelity and variability of the first distribution with respect to the second one.
Several of such variants have been proposed independently and while intuitively similar, their relationship has not yet been made explicit. In an effort to make the emerging picture of generative evaluation more clear, we propose a unification of four curves known respectively as: the precision-recall (PR) curve, the Lorenz curve, the receiver operating characteristic (ROC) curve and a special case of Rényi divergence frontiers. 
In addition, we discuss possible links between PR / Lorenz curves with the derivation of domain adaptation bounds.
\end{abstract}

\begin{keywords}
trade-off curve, distributional closeness, generative modeling, domain adaptation
\end{keywords}

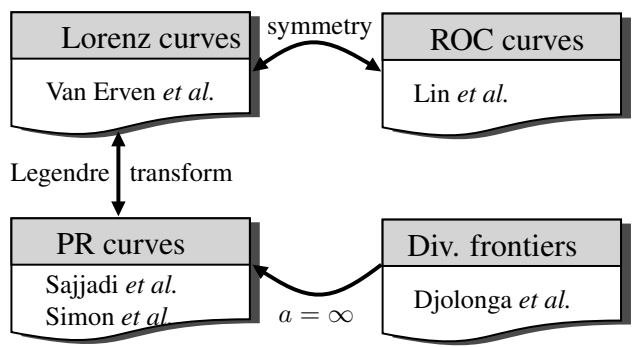
\begin{figure}[ht]
    \centering
    \tikzset{every picture/.style={line width=0.75pt}} 

\begin{tikzpicture}[x=0.75pt,y=0.75pt,yscale=-1,xscale=1]

\draw  [fill={rgb, 255:red, 255; green, 255; blue, 255 }  ,fill opacity=1 ][general shadow={fill={rgb, 255:red, 74; green, 74; blue, 74 }  ,shadow xshift=3.75pt,shadow yshift=-1.5pt, opacity=1 }] (93.54,15) -- (272.85,15) -- (272.85,93.38) .. controls (160.78,93.38) and (183.19,121.64) .. (93.54,103.35) -- cycle ;
\draw  [fill={rgb, 255:red, 212; green, 212; blue, 212 }  ,fill opacity=1 ] (93.54,15) -- (272.85,15) -- (272.85,51) -- (93.54,51) -- cycle ;

\draw  [fill={rgb, 255:red, 255; green, 255; blue, 255 }  ,fill opacity=1 ][general shadow={fill={rgb, 255:red, 74; green, 74; blue, 74 }  ,shadow xshift=3.75pt,shadow yshift=-1.5pt, opacity=1 }] (372.69,16) -- (552,16) -- (552,94.38) .. controls (439.93,94.38) and (462.34,122.64) .. (372.69,104.35) -- cycle ;
\draw  [fill={rgb, 255:red, 212; green, 212; blue, 212 }  ,fill opacity=1 ] (372.69,16) -- (552,16) -- (552,52) -- (372.69,52) -- cycle ;
\draw  [fill={rgb, 255:red, 255; green, 255; blue, 255 }  ,fill opacity=1 ][general shadow={fill={rgb, 255:red, 74; green, 74; blue, 74 }  ,shadow xshift=3.75pt,shadow yshift=-1.5pt, opacity=1 }] (93.54,175) -- (272.85,175) -- (272.85,253.38) .. controls (160.78,253.38) and (183.19,281.64) .. (93.54,263.35) -- cycle ;
\draw  [fill={rgb, 255:red, 212; green, 212; blue, 212 }  ,fill opacity=1 ] (93.54,175) -- (272.85,175) -- (272.85,211) -- (93.54,211) -- cycle ;
\draw  [fill={rgb, 255:red, 255; green, 255; blue, 255 }  ,fill opacity=1 ][general shadow={fill={rgb, 255:red, 74; green, 74; blue, 74 }  ,shadow xshift=3.75pt,shadow yshift=-1.5pt, opacity=1 }] (94.5,336) -- (273.81,336) -- (273.81,414.38) .. controls (161.74,414.38) and (184.16,442.64) .. (94.5,424.35) -- cycle ;
\draw  [fill={rgb, 255:red, 212; green, 212; blue, 212 }  ,fill opacity=1 ] (94.5,336) -- (273.81,336) -- (273.81,372) -- (94.5,372) -- cycle ;
\draw [line width=1.5]    (173.99,111.5) -- (173.99,169.5) ;
\draw [shift={(173.99,173.5)}, rotate = 270] [fill={rgb, 255:red, 0; green, 0; blue, 0 }  ][line width=0.08]  [draw opacity=0] (11.61,-5.58) -- (0,0) -- (11.61,5.58) -- cycle    ;
\draw [shift={(173.99,107.5)}, rotate = 90] [fill={rgb, 255:red, 0; green, 0; blue, 0 }  ][line width=0.08]  [draw opacity=0] (11.61,-5.58) -- (0,0) -- (11.61,5.58) -- cycle    ;
\draw  [fill={rgb, 255:red, 255; green, 255; blue, 255 }  ,fill opacity=1 ][general shadow={fill={rgb, 255:red, 74; green, 74; blue, 74 }  ,shadow xshift=3.75pt,shadow yshift=-1.5pt, opacity=1 }] (372.69,176) -- (552,176) -- (552,254.38) .. controls (439.93,254.38) and (462.34,282.64) .. (372.69,264.35) -- cycle ;
\draw  [fill={rgb, 255:red, 212; green, 212; blue, 212 }  ,fill opacity=1 ] (372.69,176) -- (552,176) -- (552,212) -- (372.69,212) -- cycle ;
\draw [line width=1.5]    (173.99,271.5) -- (173.99,333.5) ;
\draw [shift={(173.99,267.5)}, rotate = 90] [fill={rgb, 255:red, 0; green, 0; blue, 0 }  ][line width=0.08]  [draw opacity=0] (11.61,-5.58) -- (0,0) -- (11.61,5.58) -- cycle    ;
\draw  [fill={rgb, 255:red, 255; green, 255; blue, 255 }  ,fill opacity=1 ][general shadow={fill={rgb, 255:red, 74; green, 74; blue, 74 }  ,shadow xshift=3.75pt,shadow yshift=-1.5pt, opacity=1 }] (374.5,336) -- (553.81,336) -- (553.81,414.38) .. controls (441.74,414.38) and (464.16,442.64) .. (374.5,424.35) -- cycle ;
\draw  [fill={rgb, 255:red, 212; green, 212; blue, 212 }  ,fill opacity=1 ] (374.5,336) -- (553.81,336) -- (553.81,372) -- (374.5,372) -- cycle ;
\draw [line width=1.5]    (494.99,264) -- (494.99,334.5) ;
\draw [shift={(494.99,260)}, rotate = 90] [fill={rgb, 255:red, 0; green, 0; blue, 0 }  ][line width=0.08]  [draw opacity=0] (11.61,-5.58) -- (0,0) -- (11.61,5.58) -- cycle    ;
\draw [line width=1.5]    (283,63.96) -- (365.5,63.04) ;
\draw [shift={(369.5,63)}, rotate = 179.37] [fill={rgb, 255:red, 0; green, 0; blue, 0 }  ][line width=0.08]  [draw opacity=0] (11.61,-5.58) -- (0,0) -- (11.61,5.58) -- cycle    ;
\draw [shift={(279,64)}, rotate = 359.37] [fill={rgb, 255:red, 0; green, 0; blue, 0 }  ][line width=0.08]  [draw opacity=0] (11.61,-5.58) -- (0,0) -- (11.61,5.58) -- cycle    ;
\draw [line width=1.5]    (283,223.96) -- (365.5,223.04) ;
\draw [shift={(369.5,223)}, rotate = 179.37] [fill={rgb, 255:red, 0; green, 0; blue, 0 }  ][line width=0.08]  [draw opacity=0] (11.61,-5.58) -- (0,0) -- (11.61,5.58) -- cycle    ;
\draw [shift={(279,224)}, rotate = 359.37] [fill={rgb, 255:red, 0; green, 0; blue, 0 }  ][line width=0.08]  [draw opacity=0] (11.61,-5.58) -- (0,0) -- (11.61,5.58) -- cycle    ;
\draw [line width=1.5]  [dash pattern={on 1.69pt off 2.76pt}]  (283,383.96) -- (365.5,383.04) ;
\draw [shift={(369.5,383)}, rotate = 179.37] [fill={rgb, 255:red, 0; green, 0; blue, 0 }  ][line width=0.08]  [draw opacity=0] (11.61,-5.58) -- (0,0) -- (11.61,5.58) -- cycle    ;
\draw [shift={(279,384)}, rotate = 359.37] [fill={rgb, 255:red, 0; green, 0; blue, 0 }  ][line width=0.08]  [draw opacity=0] (11.61,-5.58) -- (0,0) -- (11.61,5.58) -- cycle    ;
\draw [line width=1.5]    (424.5,272) -- (424.02,331.5) ;
\draw [shift={(423.99,335.5)}, rotate = 270.46] [fill={rgb, 255:red, 0; green, 0; blue, 0 }  ][line width=0.08]  [draw opacity=0] (11.61,-5.58) -- (0,0) -- (11.61,5.58) -- cycle    ;

\draw (127.22,24) node [anchor=north west][inner sep=0.75pt]  [font=\large] [align=left] {Lorenz curves};
\draw (115.42,67) node [anchor=north west][inner sep=0.75pt]   [align=left] {\citetalias{van2010renyi}};
\draw (392.96,68) node [anchor=north west][inner sep=0.75pt]   [align=left] {\cite{lin2018pacgan}};
\draw (404.69,25) node [anchor=north west][inner sep=0.75pt]  [font=\large] [align=left] {ROC curves};
\draw (115.64,217) node [anchor=north west][inner sep=0.75pt]   [align=left] {\cite{sajjadi2018} \\ \cite{simon2019revisiting}};
\draw (123.86,184) node [anchor=north west][inner sep=0.75pt]  [font=\large] [align=left] {PR curves\\};
\draw (114.32,388) node [anchor=north west][inner sep=0.75pt]   [align=left] {\cite{pmlr-v108-djolonga20a}};
\draw (108.47,345) node [anchor=north west][inner sep=0.75pt]  [font=\large] [align=left] {Divergence frontiers};
\draw (292.18,18) node [anchor=north west][inner sep=0.75pt]   [align=left] {symmetry};
\draw (111.89,292) node [anchor=north west][inner sep=0.75pt]   [align=left] {$\displaystyle a=\infty $};
\draw (103.97,130) node [anchor=north west][inner sep=0.75pt]   [align=left] {Legendre \ \ transform};
\draw (392.21,218) node [anchor=north west][inner sep=0.75pt]   [align=left] {\cite{degroot1962uncertainty}\\\cite{liese2006divergences}};
\draw (394.19,185) node [anchor=north west][inner sep=0.75pt]  [font=\large] [align=left] {De Groot's divergence};
\draw (288.18,176) node [anchor=north west][inner sep=0.75pt]   [align=left] {$\displaystyle \pi =\frac{\lambda }{1+\lambda }$};
\draw (394.32,378) node [anchor=north west][inner sep=0.75pt]   [align=left] {\cite{csiszar1964information-theoretic}\\\cite{ali1966general}};
\draw (388.47,345) node [anchor=north west][inner sep=0.75pt]  [font=\large] [align=left] {$\displaystyle \phi -$Divergence};
\draw (319.89,284) node [anchor=north west][inner sep=0.75pt]   [align=left] {\begin{minipage}[lt]{67.94pt}\setlength\topsep{0pt}
\begin{flushright}
integral\\representation
\end{flushright}

\end{minipage}};
\draw (501.89,292) node [anchor=north west][inner sep=0.75pt]   [align=left] {$\displaystyle D_{\phi _{\pi }}$};

\end{tikzpicture}
    \caption{Overview of different works and their relationship. \revision{Note that PR and ROC curves refer here to their definition with respect to distributions not classifiers.}}
    \label{fig:related-work}
\end{figure}

\section{Introduction}\label{sec:introduction}

\revision{Assessing the proximity of two probability distributions is a long standing concern in statistics. It has gained a new impetus with the advent of deep learning techniques to generate random samples from complex data distributions such as natural images. }
Generative models, particularly Generative Adversarial Networks (GANs), can now synthesize images with unprecedented realism.
Indeed, the quality of generation has improved significantly\footnote{\footnotesize\url{https://thispersondoesnotexist.com/}} to the point where for certain datasets, human observers have difficulty discerning real and fake \citep{brock2018large,karras2019style}.
As these networks see real applications, evaluation of generative networks has become essential and remains challenging.
Additionally, such performance raises suspicion about memorizing or overfitting some training images.
The amount of training data makes visual comparative evaluation not reliable enough.
For instance, the privacy and security of generative models has become paramount, including the largest ever Kaggle competition\footnote{\footnotesize\url{https://www.kaggle.com/c/deepfake-detection-challenge}} to address deep fakes, or several new works addressing how generative models leak training data \citep{wang2019cnn,webster2019GenOverfit}. 
Even properly determining sample quality remains challenging \citep{borji2019pros}. The Fr\'echet Inception Distance (FID) \citep{heusel2017gans} was shown to correlate decently with human evaluation and remains the most popular evaluation metric, but as a scalar metric is limited when assessing model failure \citep{sajjadi2018}. A variety of other approaches attempt to give an empirical estimation of sample quality, for instance in \citep{im2018quantitatively}, the original GAN training divergence was re-used for evaluation. \citet{sajjadi2018} proposed computing an entire precision-recall curve for the generated distribution. Unlike the scalar FID, this curve distinguishes what we will refer to as the {\em fidelity} and {\em variability} of the model \citep{naeem2020reliable}. Fidelity evaluates whether the generated distribution produces data that are faithful to the original distribution whereas variability reflects the fact that it covers the entire distribution with the correct importance. For instance, a generator of facial images with poor diversity may only generate one gender, whereas a generator with poor fidelity contains generation artefacts.

\par

Even if the question of estimating the similarity of two distributions has already been intensely studied in the past, it has drawn a renewed interest in recent years. Many probability distribution metrics can be used such as divergences \citep{rubenstein2019practical} or integral probability metrics \citep{muller1997integral,sriperumbudur2009integral} to name a few. In this work, we will take special focus on the recent work of \citet{sajjadi2018} computing a Precision-Recall (PR) curve between two distributions\footnote{\revision{Note that as we shall detail later, PR curves between distributions are different from the classical notion of PR curves in classification.}}.  
Several works explicitly build upon their definition and propose some extensions.
For instance, \citet{simon2019revisiting} generalizes the PR curve to arbitrary probabilities (while the work of \citet{sajjadi2018} was restricted to a discrete settings).
More practical works aim at improving empirical evaluation of fidelity and variability \citep{kynkaanniemi2019improved,naeem2020reliable}.
\citet{djolonga2019evaluating} proposed the (R\'enyi) {\em divergence frontiers}; this alternate curve coincides with the original PR curve for discrete distributions when the R\'enyi exponent is infinite.
 Independently, a handful of alternative curves were defined to compare two distributions.
For instance, {\em ROC curves} were proposed by \citet{lin2018pacgan, lin2017arxiv} and the {\em Lorenz curves} by \citet{harremoes2004new,van2010renyi}.

In this work, we demonstrate that, despite their apparently independent definitions, these alternate notions are actually tightly linked with the PR curves.
Our main contribution 
is the theoretical unification between the involved curves, which is summarized in the diagram of Fig.~\ref{fig:related-work}.
After briefly introducing standard hypotheses and notations, we recall definitions and properties of the aforementioned curves in Section~\ref{sec:related_works}.
Relations between them are scrutinized in Section~\ref{sec:link}.
In particular, we  consider the link found by \citet{djolonga2019evaluating} between PR curves and divergence frontiers (\SectionSymbol~\ref{sec:DivFrontiers-PR}) for infinite R\'enyi exponent $a$, hence extending it from discrete distributions to general ones. 
More importantly, we show that Lorenz curves and PR curves are related through convex duality (\SectionSymbol~\ref{sec:equivalence}). 
As a side contribution, we end these notes in Section~\ref{sec:link_DA} by highlighting links existing between the aforementioned curves and performance bounds used in the theory of domain adaptation. 
\revision{Last, in Section~\ref{sec:phi-div}, we explore several links with $\phi$-divergences. In particular, thanks to integral representations of $\phi$-divergences, we can show a reversed link between PR curves and divergence frontiers. Besides, starting from the variational form of $\phi$-divergence we extend the notion of Lorenz curves and use them to establish a new generalization bound for domain adaptation.}

\section{Background on trade-off curves}
\label{sec:related_works}
In this section, we review several curves proposed in the literature to assess the similarity between two distributions $P$ and $Q$. 
We simply recap the principal definitions and useful results. Some notions are subject to minor adaptations in order to simplify the exposition of the links between the considered curves. Anytime such a revision is adopted, it shall be explicitly mentioned.

Let us start by recalling some standard notations, definitions, and results from measure theory.
From now on, $(\measSpace,\mathcal A)$ represents a common measurable space, and we will denote $\measSet$ the set of sigma-finite signed measures, $\pmeasSet$ the set of sigma-finite positive measures and $\probaSet$ the set of probability distributions over that measurable space.  The extended half real-line  is denoted by $\overline{\R^+} = \R^+ \cup \{\infty\}$.
\begin{definition}
Let $\mu,\nu$ two signed measures. We denote by 
    $\mathrm{supp}(\mu)$ the support\footnote{Although a precise definition of the support requires a topology, we gloss over this issue because the support will not play a central role in the technical derivations.} of $\mu$,
    $|\mu|$ the total variation measure of  $\mu$,
    $\frac{d\mu}{d\nu}$ the Radon-Nikodym derivative of $\mu$ w.r.t. $\nu$ and
    $\mu\wedge\nu = \min(\mu,\nu) := \frac 12 (\mu+\nu -|\mu-\nu|)$
    (a.k.a the measure of largest common mass between $\mu$ and $\nu$ \citep{piccoli2019wasserstein}).
    Besides, as usual, $\mu\ll\nu$ means that $\mu$  is absolutely continuous w.r.t. $\nu$.
\end{definition}

\subsection{Precision-recall curves}
The PR curves were first proposed by \cite{sajjadi2018} for discrete distributions and then extended to the general case by \cite{simon2019revisiting}. We follow the definition of the latter up to a minor fix\footnote{There is an issue with their original definition where $(1,0)$ and $(0,1)$ are always in $\PRD(P,Q)$, while they should not when part of the mass of $P$ is absent from $Q$ and vice versa. Our fix consists in considering only the distributions $\mu$ that are absolutely continuous w.r.t $P$ and $Q$.}.
\begin{definition}
Let $P,Q$ two distributions from $\probaSet$. We refer to the Precision-Recall set $\PRD(P,Q)$ as the set of Precision-Recall pairs $(\alpha,\beta)\in\R^+\times\R^+$ such that
\begin{equation}
 \exists\, \mu \in\AC(P,Q),  P\geq\beta \mu , Q \geq \alpha \mu
 \;,
\end{equation}
where $\AC(P,Q):=\{\mu\in\probaSet / \mu\ll P \text{ and } \mu\ll Q\}$.
\end{definition}
The \emph{precision} value $\alpha$ is related to the proportion of the generated distribution $Q$ that matches the true data $P$, while conversely the \emph{recall} value $\beta$ is the amount of the distribution $P$ that can be reconstructed from $Q$.
Because of the lack of natural order on $[0,1]\times[0,1]$, \cite{simon2019revisiting} has proposed to focus on the Pareto front of $\PRD(P,Q)$ defined as follows. %
\begin{definition}
\label{def:prd-curve}
The precision recall-curve $\partial \PRD(P,Q)$ is the set of $(\alpha,\beta)\in \PRD(P,Q)$ such that
    \begin{equation*}
       \forall (\alpha',\beta')\in \PRD(P,Q), \alpha\geq\alpha'\text{ or }\beta\geq\beta' .
    \end{equation*}
    \revision{
    Note that this curve is clearly the Pareto front of the set :
    \begin{equation*}
        \{ (\kappa^*(Q|\mu), \kappa^*(P|\mu)) / \mu\in \AC(P,Q)\}
    \end{equation*}
    where, following \citet{scott2013classification}, we define $\kappa^*(P|\mu):=\max\{\alpha\in [0,1] / \exists \nu\in\probaSet, P = \alpha \mu +(1-\alpha)\nu\} = \inf_{\substack{A\in\mathcal A\\ \mu(A)>0}} \frac{P(A)}{\mu(A)}$ (and similarly for $Q$).
    }
\end{definition}

In fact, this frontier is a curve for which \cite{sajjadi2018} have exposed a parameterization, later generalized by \cite{simon2019revisiting}. We recall their result now.
\begin{theorem}
\label{thm:dprd}
Let $P,Q$ two distributions from $\probaSet$
and $(\alpha,\beta)$ non negative. Then, denoting
    \begin{equation}
    \forall \lambda \in \overline{\R^+},
    \left\{
    \begin{aligned}
        \alpha_\lambda :=& \left({(\lambda P)\wedge Q} \right) (\measSpace)\\
        \beta_\lambda :=& \left({P\wedge\tfrac 1\lambda Q}\right)(\measSpace)
    \end{aligned}
    \right.
    \end{equation}
\begin{enumerate}
    \item $(\alpha,\beta)\in \PRD(P,Q)$ iff $\alpha\leq \alpha_\lambda$ and $\beta \leq \beta_\lambda$ where $\lambda:=\frac \alpha\beta\in\overline{\R^+}$.
    \item  As a result,  the PR curve can be parameterized as:
\begin{equation}
    \partial \PRD(P,Q) = \lbrace (\alpha_\lambda, \beta_\lambda) / \lambda \in \overline{\R^{+}}\rbrace
    \, .
\end{equation}

\end{enumerate}
\end{theorem}
In the previous theorem, and in agreement with the standard convention in measure theory $0\times\infty=0$ so that $\alpha_\infty=Q(\mathrm{supp}(P))$ and $\beta_0=P(\mathrm{supp}(Q))$.

\revision{
\begin{remark}
\label{rmk:groot}
The previous parameterization reveals that the precision-recall curve is intrinsically related to the De Groot statistical information \citep{degroot1962uncertainty} which is defined  as $\Delta B_\pi(P,Q):= B_\pi(P,P)-B_\pi(P,Q)$ through the following divergence (that we will refer in short as the De Groot divergence):
\begin{equation}
    \label{eq:degroot-bpi}
   B_\pi(P,Q):=[\pi P \wedge (1-\pi) Q](\measSpace) 
\end{equation}
where $\pi\in [0,1]$ is an arbitrary prior probability. In simple words, the link with the precision recall curve is merely a change of parameterization :
$\pi =\frac{\lambda}{1+\lambda}$. 
This relationship will play a crucial role in one of the links made with $\phi$-divergences in section~\ref{sec:phi-div} (see the integral representation arrow in Figure~\ref{fig:related-work}).
\end{remark}
}

Another useful result of \cite{simon2019revisiting} linking the frontier to likelihood ratio test is summarized now:
\begin{theorem}
\label{thm:dprd2}
Let $P,Q$ two distributions from $\probaSet$. Then 
    \begin{equation}
    \label{eq:dprd2}
    \forall \lambda \in \overline{\R^+},
    \left\{
    \begin{aligned}
        \alpha_\lambda =& \lambda \left(1-P(\LRS)\right) + Q(\LRS)\\
        \beta_\lambda =& 1-P(\LRS) + \tfrac{Q(\LRS)}\lambda\\
    \end{aligned}
    \right.,
    \end{equation}
    where the likelihood ratio sets are defined as
    \begin{equation}
        \label{eq:lrs}
        \LRS:=\left\lbrace \tfrac{dQ}{d(P+Q)} \leq \lambda \tfrac{dP}{d(P+Q)} \right\rbrace
        \, .
    \end{equation}
\end{theorem}

\subsection{Divergence frontiers}

Divergence frontiers were proposed very recently by \cite{djolonga2019evaluating} as a generalization of precision-recall curves. Such a notion builds upon the R\'enyi divergence between two distributions.

\begin{definition}
\label{def:renyi}
    Let $\mu,\nu\in\probaSet$ two distributions such that $\mu\ll\nu$ and $a\in\overline{\R^+}\setminus\{1\}$. The $a$-R\'enyi-divergence between $\mu$ and $\nu$ is defined as:
    \begin{equation}
        D_a(\mu\parallel \nu):=\log\left(\left\Vert \tfrac{d\mu}{d\nu} \right\Vert_{a-1, d\mu}\right)
    \end{equation}
    where the invoked norm is defined as 
    $$\Vert  f \Vert_{a-1, d\mu}:= \left(\int f^{a -1} d\mu \right)^{\tfrac 1{a-1}}$$ 
    when $a<+\infty$ and is the essential supremum norm for $a=\infty$.
    Besides when $a=1$, this definition is extended by continuity and leads to the KL-divergence.
\end{definition}

We adapt the definition of divergence frontiers from \cite{djolonga2019evaluating}.
\begin{definition}[Divergence frontiers] Let $P,Q$ two distributions and $a\in\overline{\R^+}$. Then the exclusive realizable divergence region is defined\footnote{In the original work, the distribution $\mu$ may range on a restricted set of distributions such as an exponential family. Besides, to avoid situations where the divergence is ill-defined, we imposed that $\mu$ is absolutely continuous w.r.t both $P$ and $Q$.} as the set:
\begin{equation}
    \mathcal R_a^\cap(P,Q) := \left\lbrace \left(D_a(\mu\parallel Q), D_a(\mu\parallel P)\right) / \mu\in\AC(P,Q)\right\rbrace
\end{equation}
And the exclusive divergence frontier is defined as the (weak) Pareto front of this region, that is the set $\partial \mathcal R_a^\cap$ of couples $(\pi,\rho)\in \mathcal R_a^\cap$ such that:
\begin{equation}
 \forall (\pi',\rho')\in \mathcal R_a^\cap, \pi \leq \pi'\text{ or } \rho\leq\rho'
    \, .
\end{equation}
In the event that $\mathcal R_a^\cap(P,Q)=\emptyset$, the frontier is by convention restricted to the point $(+\infty,+\infty)$.
\end{definition}

\subsection{Lorenz and ROC curves}\label{sec:Lorenz-ROC}

Lorenz curves were originally introduced by \cite{lorenz1905methods} to delineate income inequalities. In essence they highlight how much a single one dimensional distribution differs from a uniform distribution. This notion was then generalized to characterize the closeness of two arbitrary distributions by \cite{harremoes2004new,van2010renyi}. 

\begin{definition} Let $P,Q$ two distributions from $\probaSet$. 
One defines the Lorenz diagram between $P$ and $Q$ as 
\begin{equation}
\label{eq:ld}
   \LD(P,Q)=\left\lbrace \left(\int f dP, \int f dQ\right) / 0 \leq f \leq 1 \right\rbrace,
\end{equation}
where the function $f$ is required to be measurable.

Then the Lorenz curve between $P$ and $Q$ is defined as the lower envelop of the Lorenz diagram:
\begin{equation}
\label{eq:dld}
   \dLD^{P,Q}(t):= \inf_{\substack{0\leq f \leq 1 \\ \int f dP\geq t}} \int f dQ.
\end{equation}
In absence of ambiguity on the involved distributions, we shall denote it simply $F(t)$ rather than $\dLD^{P,Q}(t)$.
This curve is easily shown to be a monotonic and convex function. 
\end{definition}

If one considers in \eqref{eq:ld} only the range of indicator functions, then one recovers a subset of the Lorenz diagram, from which the Lorenz diagram can be extracted by merely taking the closed convex hull. This fact underpins the equivalence between Lorenz diagrams/curves and the seemingly different notions of Mode Collapse Region / ROC curves proposed by \cite{lin2018pacgan}. 
Indeed, they  show in \cite{lin2017arxiv}  (Remark~6) that their Mode Collapse Region (MCR) can be obtained as the convex hull of the set of points $(P(A), Q(A))$ where $A$ is any measurable set such that $Q(A)\geq P(A)$.
As such the MCR is the upper half of the Lorenz diagram when one cuts it along the main diagonal (i.e. the line segment between $(0,0)$ and $(1,1)$).
Then the authors proceed to define the ROC curve as the upper envelop of the MCR, which in turns is the symmetric transform of the lower envelop (i.e. the Lorenz curve) along the same diagonal.
For the sake of time precedence, we shall thereafter focus solely on the Lorenz curve.

Similarly, restricting \eqref{eq:dld} to indicator functions requires convexification (more precisely $\Gamma$-regularization) to recover the Lorenz curve. In fact, because of the Neyman-Pearson lemma, one can even restrict to the indicator functions of the likelihood ratio sets $\LRS$,
which in light of Thm~\ref{thm:dprd2} underlies a subtle link with precision-recall curves that we shall detail later on.

\section{Relationships between trade-off curves}\label{sec:link}
\begin{figure}[ht]
    \centering
    \begin{tabular}{c}
        \includegraphics[width=0.7\columnwidth]{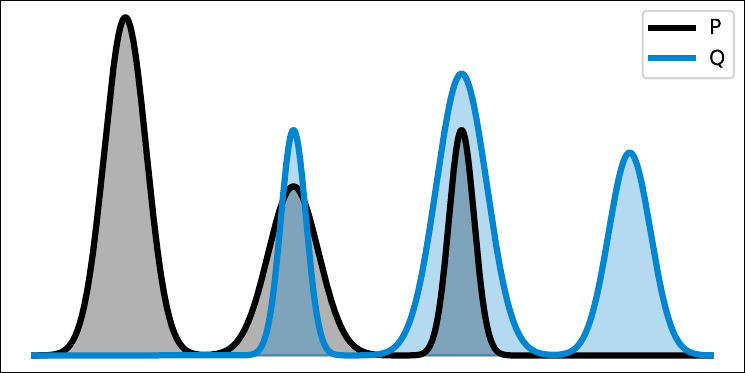}
    \end{tabular}
    
    \centering
    \begin{tabular}{ccc}
        \includegraphics[width=.3\columnwidth]{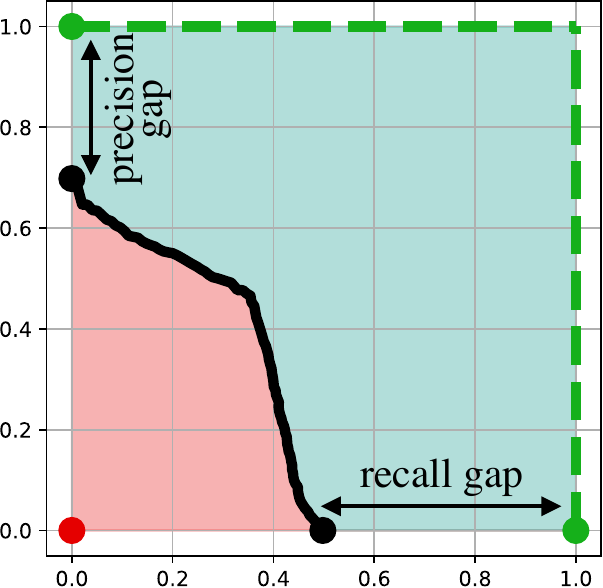}
        & \hspace{0.1\columnwidth}
        & \includegraphics[width=.3\columnwidth]{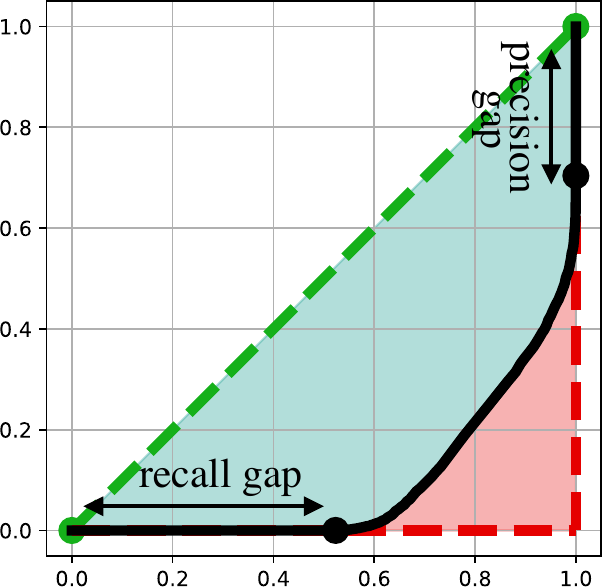}
        \\
        Precision-Recall Curve
        & 
        & Lorenz Curve
        \\
    \end{tabular}
    \caption{top: graphical representations of two mixtures of Gaussians $P$ and $Q$. bottom: the corresponding alternate similarity curves. For each alternative, the green dashed curve exposes the extreme case where $P=Q$ and the red spot/curve when $P\perp Q$.}
    \label{fig:complex}
\end{figure}

\begin{figure}[ht]
    \centering
    \begin{tabular}{c}
    \end{tabular}
    
    \centering
    \setlength{\tabcolsep}{2pt}
    \renewcommand{\arraystretch}{2.1}
    \begin{tabular}{cccc}
        \raisebox{5mm}{\rotatebox[origin=c]{90}{$P$ \& $Q$}}
        &\includegraphics[width=0.3\columnwidth]{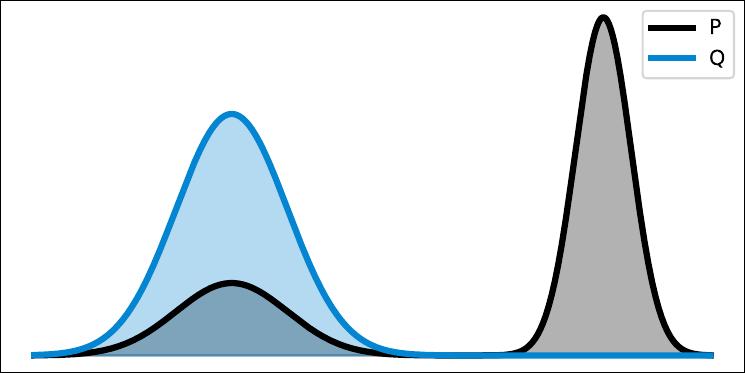}
        &\includegraphics[width=0.3\columnwidth]{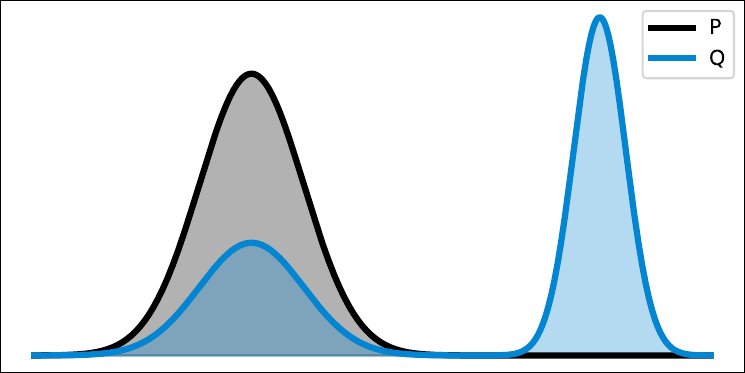}
        &\includegraphics[width=0.3\columnwidth]{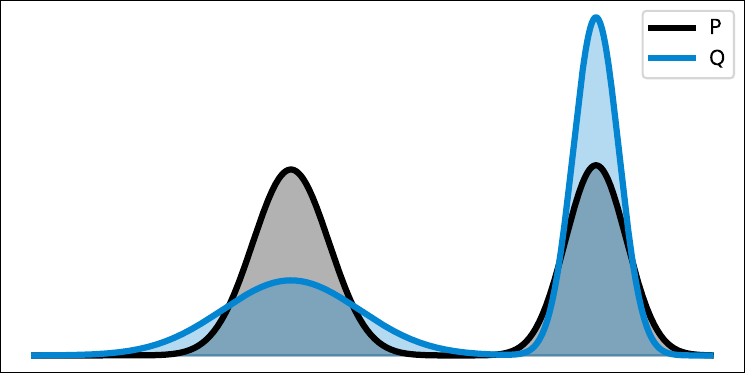} 
        \\[2mm]
        \raisebox{10mm}{\rotatebox[origin=c]{90}{PR curves}}
        &\includegraphics[width=.25\columnwidth]{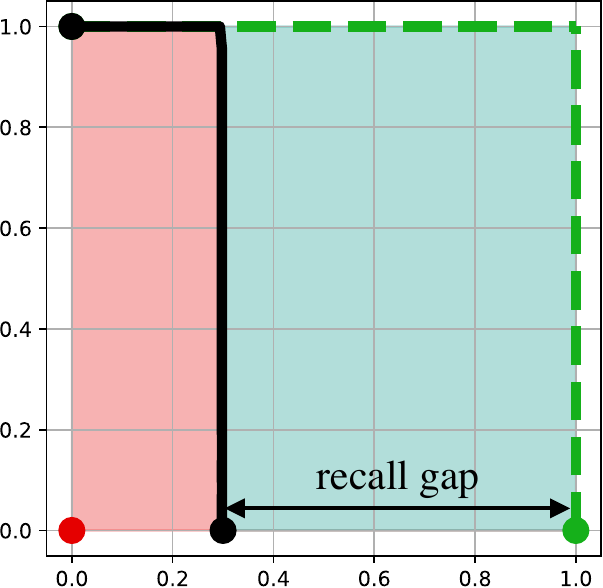}
        & \includegraphics[width=.25\columnwidth]{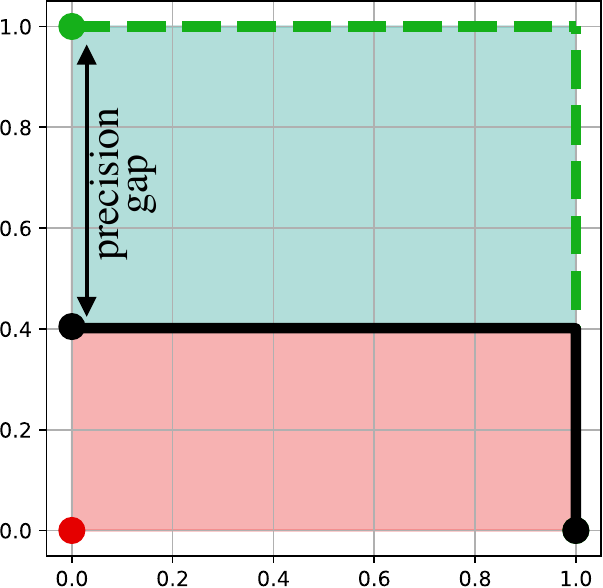}
        & \includegraphics[width=.25\columnwidth]{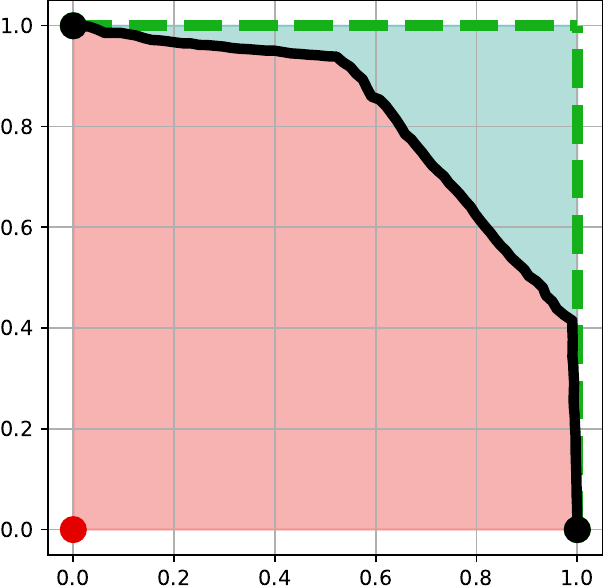}
        \\[2mm]
        \raisebox{10mm}{\rotatebox[origin=c]{90}{Lorenz curves}}
        &\includegraphics[width=.25\columnwidth]{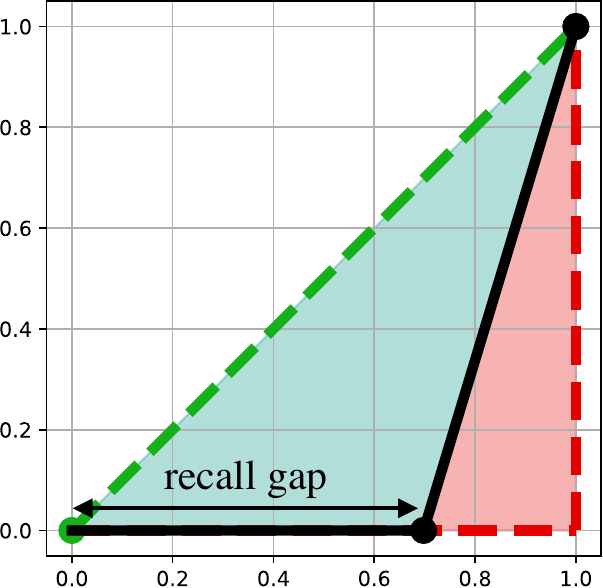}
        & \includegraphics[width=.25\columnwidth]{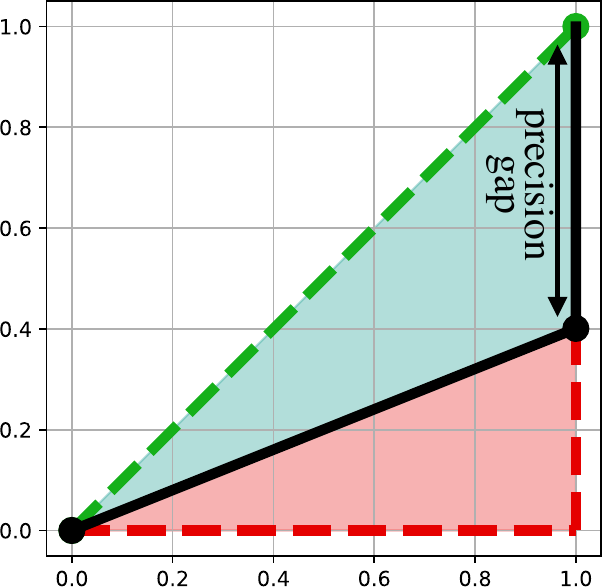}
        & \includegraphics[width=.25\columnwidth]{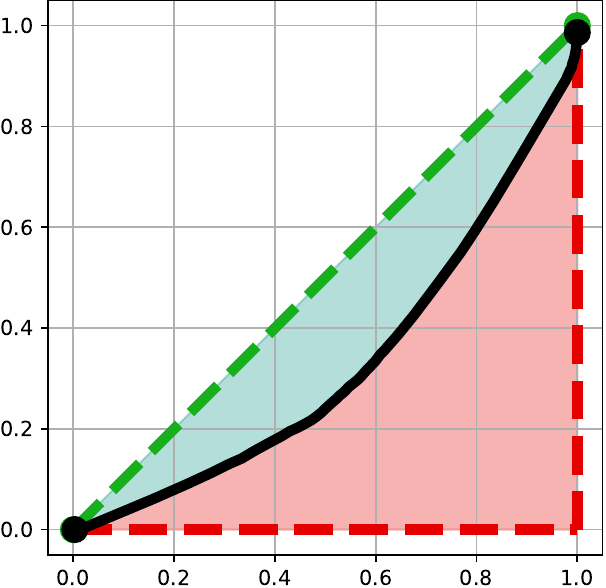}
        \\
        & (a)
        & (b)
        & (c)
        \\
    \end{tabular}
    \caption{PR and Lorenz Curves in different scenarios: (a) pure mode dropping -- (b) pure mode invention -- (c) mode reweighting. Top row : the two distributions $P$ (in black) and $Q$ (in blue). Middle row: PR curves. Bottom row: Lorenz curves.}
    \label{fig:cases}
\end{figure}

Before going into further details about how the aforementioned curves relate to one another, 
let us state a few general facts about how they differ. 
One can first note that each curve is subject to specific ``regularity'' properties such as monotonicity, convexity and boundedness. 
For instance, contrary to the Lorenz curve, the PR curve does not enjoy any convexity property.
Similarly, both the PR curve and the Lorenz curve are bounded within the domain $[0,1]\times[0,1]$, while the divergence frontiers are not bounded in general. 
Last both the divergence frontiers and PR curves are decreasing while the Lorenz curve is increasing.

Despite these disparities, each curve serve a similar purpose and strong links exist between them. 
In the next paragraph, a few simple situations are first scrutinized to describe the behavior of PR and Lorenz curves (divergence frontiers are voluntarily excluded from this preamble).
Then, in \SectionSymbol~\ref{sec:DivFrontiers-PR} and \SectionSymbol~\ref{sec:equivalence}, we shall establish the exact links between those curves.

\subsection{Some intuition on the PR and Lorenz curves}\label{sec:Lorenz-PR}
 
In this preamble, we focus on the principal function of the curves under consideration: namely how they characterize the similarity between $P$ and $Q$.
To make our discussion more concrete, we consider an illustrative case in Fig.~\ref{fig:complex}, where $P$ and $Q$ are two Gaussian mixtures. 
For a given curve alternative, one may consider two extreme configurations. 
One one hand, the perfect match between $P$ and $Q$ \emph{i.e.} $P=Q$ is represented in dashed green. 
On the other hand, the complete discord between $P$ and $Q$, denoted by $P\perp Q$ corresponds to an empty overlap of their supports (or more formally to two mutually singular distributions) and is represented in red.
Then a particular instance of the considered curve will appear as an in-between case. The closer it stands to the green spot (and therefore the farther from the red spot), the more similar $P$ and $Q$.

To illustrate the benefit of a trade-off curve in comparison to scalar metrics, we depict a few examples in Fig.~\ref{fig:cases}. The three examples are obtained by adjusting the locations, widths and weights of the Gaussian mixture models. 
They corresponds to scenarios of idealistic modes of deviations between the two distributions $P$ and $Q$, namely (a) pure mode dropping, (b) pure mode invention and (c) pure mode reweighting. 
In (a) a gaussian component from $P$ is missing in $Q$, and this translates in both curves. In the PR curve, it becomes manifest through a drop of recall that is depicted by the horizontal gap separating the curve from the dash green curve. In the Lorenz curve, it is depicted by the horizontal gap between the point where the curve becomes positive and the origin $(0,0)$. 
In (b) an extra gaussian component is present in $Q$ and again this phenomenon is patent in the curves,  but this time it is depicted by vertical gaps.
In (c) $P$ and $Q$ present both two components centered at the same locations, but they have different mixing factors, and one of the two components is more spread in $Q$. 
In that scenario, maximal recall and maximal precision can be achieved but not at the same time. This trade-off is rendered in a  way specific to each curve. In both cases, the horizontal and vertical gaps indicative of mode dropping and mode invention are null, but the curve smoothly interpolate from full recall to full precision away from the green curve. 
The fact that one gaussian component is identical in $P$ and $Q$ translates clearly in the precision-recall curve: indeed full recall can be obtained for a non-null precision (approximately $0.4$ in this plot).
Figure~\ref{fig:complex} is depicting an aggregate of these 3 extreme scenarios and both trade-off curves reflect the three phenomenons.

\subsection{Precision-recall vs Divergence frontiers}\label{sec:DivFrontiers-PR}
In \cite{djolonga2019evaluating}, it is shown that in the case of discrete measures, the notion of divergence frontier matches precision-recall curve in the limit case where the R\'enyi exponent $a\to\infty$. We extend here this result to the case of general distributions. To do so, we will rely on the following technical lemma.
\begin{lemma}
\label{prop:sup-ratio}
Let $\mu\in\probaSet$ and $\nu\in\probaSet$ two distributions. 
\begin{enumerate}
    \item If $\mu\ll \nu$ then
$ \sup_{A\in\mathcal A} \tfrac{\mu(A)}{\nu(A)} = \esssup_{d\mu} \tfrac{d\mu}{d\nu};$
    \item otherwise,
$$ \sup_{A\in\mathcal A} \tfrac{\mu(A)}{\nu(A)} = \esssup_{d\mu} \tfrac{d\mu}{d(\mu+\nu)} / \tfrac{d\nu}{d(\mu+\nu)}.$$
\end{enumerate}
\end{lemma}

\begin{proof}
For the sake of completeness, we provide a proof of this technical result in Appendix~\ref{app:sup-ratio}.
\end{proof}

\begin{theorem} 
\label{thm:pr-div}
Let $P,Q$ two distributions. Then,
$$\partial PRD(P,Q) = \left\lbrace (e^{-\pi}, e^{-\rho}) / (\pi,\rho)\in \partial \mathcal R_\infty^\cap(P,Q)\right\rbrace$$
\end{theorem}

\begin{proof}
From the definition of the R\'enyi divergence it is clear that 
\begin{equation*}
\begin{split}
D_\infty(\mu\parallel Q) =& \log\left(\Vert \tfrac{d\mu}{d(\mu+Q)} / \tfrac{dQ}{d(\mu+Q)} \Vert_{\infty, d\mu}\right)\\
=&  \log\left(\esssup_{d\mu}{\tfrac{d\mu}{d(\mu+Q)} / \tfrac{dQ}{d(\mu+Q)}}\right)
\end{split}
\end{equation*}
which in turn can be expressed thanks to Lemma~\ref{prop:sup-ratio} as
$$D_\infty(\mu\parallel Q) = \log\left(\sup_{A\in\mathcal A} \tfrac{\mu(A)}{Q(A)}\right) \,.$$
Similarly, 
$$D_\infty(\mu\parallel P) = \log\left(\sup_{A\in\mathcal A} \tfrac{\mu(A)}{P(A)}\right) \,.$$
Besides, \revision{as stated at the end of Definition~\ref{def:prd-curve}} it is clear that the precision-recall curve is obtained as the Pareto-front of the set of pairs
$(\inf_{\substack{A\in\mathcal A\\ \mu(A)>0}} \tfrac{Q(A)}{\mu(A)}, \inf_{\substack{A\in\mathcal A\\ \mu(A)>0}}\tfrac{P(A)}{\mu(A)})$ where $\mu$ describes all distributions in $AC(P,Q) \cap \probaSet$.
Note that because of the standard measure theory convention $\tfrac 00 =0$, we have that 
$$\sup_{A\in\mathcal A} \tfrac{\mu(A)}{Q(A)} = \sup_{\substack{A\in\mathcal A\\ \mu(A)>0}} \tfrac{\mu(A)}{Q(A)} = 1 / \inf_{\substack{A\in\mathcal A\\ \mu(A)>0}} \tfrac{Q(A)}{\mu(A)} 
= e^{D_\infty(\mu\parallel Q)}\,.
$$
The claimed identity easily follows.
\end{proof}

\subsection{Precision-recall vs Lorenz curves}
\label{sec:equivalence}
The question of the relation between PR-curves and Lorenz/ROC curves is reminiscent of the comparison of PR and ROC curves for binary classification \cite{davis2006relationship}. 
Note however that despite their name, PR-curves for distributions are not the same as the PR-curves of the likelihood ratio classifier as one might be inclined to believe.
In fact, they are composed of mixed error rates of the said classifier (see Thm~\ref{thm:dprd2}).

In essence, PR-curves and Lorenz curves are two ways of exposing the pairs $(P(\LRS), Q(\LRS))$. Yet, the following questions are not trivial. Given the PR-curve of $P$ and $Q$, can we compute their Lorenz curve? 
Reciprocally, can we compute the PR-curve from the Lorenz one? 
If one had a more complete representation such as $(\lambda, P(\LRS), Q(\LRS))$, then one could easily compute both the PR-curve and the Lorenz curve, but in each representation, at least one datum is not explicitly known:
\begin{enumerate}
    \item In the Lorenz curve, $\lambda$ is not readily available, but we will see that it is in fact hidden in the sub-derivative of the Lorenz curve.
    \item In the PR-curve, $\lambda$ can be easily computed as the ratio $\tfrac{\alpha_\lambda}{\beta_\lambda}$ but the values of $P(\LRS)$ and $Q(\LRS)$ are mingled within $\alpha_\lambda$ so that one needs to untangle them before recovering the Lorenz curve.
    Note that, given a fixed $\lambda$, the system of equations given by $\alpha_\lambda$ and $\beta_\lambda$ in Eq~\ref{eq:dprd2} is always under-determined (rank 1) and therefore does not suffice on its own to recover the values of $P(\LRS)$ and $Q(\LRS)$.
\end{enumerate}

We will rely on the following Lemma. 
\begin{lemma}
\label{cor:optimality}
Let $P,Q$ two distributions from $\probaSet$. Then 
    $\forall \lambda \in \overline{\R^+},$
    \begin{equation}
    \label{eq:alpha-optimality}
    \begin{split}
        \alpha_\lambda =& \min_{0\leq f \leq 1} \lambda (1-\int f dP) + \int f dQ\\
    \end{split}
    \end{equation}
    where the functions $f$ are measurable.
\end{lemma}
\begin{proof}
See Appendix~\ref{app:optimality}.
\end{proof}

From Lemma~\ref{cor:optimality}, one can draw the following link between the PR-curve and the Lorenz curve.

\begin{theorem}
\label{thm:duality}
Let $P$ and $Q$ two distributions. Let $\lambda\in\R^+$.
\iffalse
    Then, $$F^*(\lambda)=\lambda -\alpha_\lambda$$.
\else
    Consider the Lorenz Curve $F$ defined in Eq.~\eqref{eq:dld}, 
    then,
    \begin{equation}
    F^*(\lambda) =\lambda -\alpha_\lambda
    \end{equation} 
\fi
where $F^*(\lambda)=\sup_{t\in[0,1]}\lambda t - F(t)$ is the Legendre transform of $F$.
\end{theorem}

\begin{proof}{\ }
    Let $\lambda\geq 0$.  Let us show that $F^*(\lambda)=\lambda -\alpha_\lambda$.
    Indeed , $\forall t\in [0,1]$
    \begin{equation*}
        \begin{split}
            \lambda t - F(t) =& \lambda t - \inf_{\substack{0\leq f \leq 1 \\ \int f dP\geq t}} \int f dQ =\sup_{\substack{0\leq f \leq 1 \\ \int f dP\geq t}} \lambda t- \int f dQ\\
            \leq& \sup_{\substack{0\leq f \leq 1 \\ \int f dP\geq t}} \lambda\int f dP - \int f dQ \\
            \leq& \sup_{0\leq f \leq 1} \lambda\int f dP - \int f dQ \\
            =& \lambda - \alpha_\lambda\hspace{1cm}\text{ (thanks to Lemma~\ref{cor:optimality})}\\
        \end{split}
    \end{equation*}
    Which shows that $\lambda - \alpha_\lambda \geq \sup_{t\in[0,1]} \lambda t - F(t)$. Besides,
     letting $t_\lambda:=P(\LRS)$
    \begin{equation*}
    \begin{split}
        \lambda -\alpha_\lambda =& \lambda - (\lambda(1-P(\LRS)) + Q(\LRS) )\\
        =& \lambda P(\LRS) - Q(\LRS) = \lambda t_\lambda - F(t_\lambda)\\
    \end{split}
    \end{equation*}
    where we have used that if $t_\lambda=P(\LRS)$ then $F(t_\lambda)=Q(\LRS)$ (a result induced by the standard Neyman-Pearson lemma). 
    Therefore, $\lambda - \alpha_\lambda = \sup_{t\in[0,1]} \lambda t - F(t) = F^*(\lambda)$.
\end{proof}

\begin{remark}
Theorem~\ref{thm:duality} brings many valuable prospects concerning the link between the PR and Lorenz curves.
\begin{enumerate}
    \item First, since the Legendre transform is a one-to-one involution, the PR and Lorenz curves are theoretically equivalent.
    \item Besides, letting  $t_\lambda:=P(\LRS)$ and relying on the Fenchel identity, one gets that $\lambda\in \partial F(t_\lambda)$, which theoretically provides a means to 
    extract the missing datum as soon as one is capable of computing the subdifferential of the Lorenz curve.
    \item More concretely, the theorem provides us a practical way to compute $\alpha_\lambda$ from the Lorenz curve.
    Indeed, given $\lambda$, $\alpha_\lambda$ can be computed by solving the following 1D convex problem:
    $$\alpha_\lambda = \lambda - F^*(\lambda) = \min_{t\in[0,1]} F(t)+\lambda(1-t)$$
    One can do so efficiently thanks to the bisection method if the subdifferential of $F$ is available or resort to derivative-free algorithms such as the Golden Section Search method otherwise.
    Then $\beta_\lambda$ is obtained as $\tfrac{\alpha_\lambda}{\lambda}$.
    \item In the other way around, given $t\in[0,1]$, one can solve for $F(t)$ by considering the following 1D concave problem:
    \begin{equation*}
    \begin{split}
    F(t) = F^{**}(t) &= \sup_{\lambda\in\R^+}\lambda t - F^*(\lambda) \\
    &= \sup_{\lambda\in\R^+} \alpha_\lambda + \lambda(t-1).
    \end{split}
    \end{equation*}
\end{enumerate}
\end{remark}

\section{Links with domain adaptation}\label{sec:link_DA}
In this part we revisit a standard bound from domain adaptation. In this setting, $P$ and $Q$ stand for the source and target distributions over a joint space $\measSpace = \mathcal X\times \mathcal Y$ where $\mathcal X$ is called the sample space and $\mathcal Y$ the label space and is a finite set of classes. Given a classifier $h$ (a.k.a an hypothesis) one would like to control the error $\risk_Q(h)$ in the target domain which is not directly observable (for lack of supervision) with a bound relying on the corresponding error $\risk_P(h)$ in the source domain (where supervision is available). The simplest bound proposed in the literature is the following \citep{ben2010theory}:
\begin{equation}
    \label{eq:da-tv}
    \risk_Q(h):=\int \1_{h(x)\neq y} dQ(x,y)
    \leq
    \risk_P(h) + \Vert P-Q\Vert_{TV}
\end{equation}

In the common {\em covariate-shift} assumption (where the distribution of labels given samples is the same in both domains), the bound can be expressed in terms of marginal distribution over $\mathcal X$.
It was noted directly by \cite{ben2010theory} that the previous bound suffers from two major issues in practice. The first one is that the bound cannot be estimated efficiently from finite i.i.d samples because of its reliance on the TV norm (which involves a class of measurable sets that is not restricted enough for such purpose). 
And the second issue lies in the fact that the bound does not imply the class of hypothesis functions where $h$ lives (e.g. a class of small VC dimension or of small Rademacher capacity).
As such, the bound in question is over-pessimistic, because it takes for granted that error set $\{h(x)\neq y\}$ may be spread as an arbitrary measurable set, which may be far from the case depending on restrictions applying to $h$.
The authors then derived a more adapted bound based on the so-called $\mathcal H \Delta \mathcal H$-divergence, answering the previous two points, and many other bounds were later proposed. The interested reader may refer to \cite{redko2020survey} for a very up-to-date and comprehensive review of such works.

\subsection{A refined bound}
Despite its shortcomings, we would like to revisit Eq.~\ref{eq:da-tv} and provide an optimized version of it which has an intuitive interpretation.
We will first derive a bound, based on the Lorenz curve between $P$ and $Q$ and then express it in a form closer to Eq~\ref{eq:da-tv} by relying on the duality between Lorenz and PR curves.

\begin{proposition}\label{prop:error_bounds_Lorenz}
The Lorenz curve provides a lower bound for domain adaptation.
\begin{equation}
    \risk_Q(h)\leq 1-F(1-\risk_P(h))
\end{equation}
\end{proposition}
\begin{proof}

\revisioncomment{It is hard to implement the remark of R\#2 here, because the constraint set in $\inf$ is different every time.}
\begin{equation*}
    \begin{split}
    F(1-\risk_P(h))
        & = \inf_{\substack{g \text{ measurable}\\0\leq g \leq 1 \\\int g d\S \ge 1 - \int  \1_{h\neq f} d\S}} \int g d\T
        \\
        & = \inf_{\substack{g \text{ measurable}\\0\leq 1-g \leq 1 \\\int (1 - g) d\S \le \int  \1_{h\neq f} d\S}} \int 1 - (1-g) d\T
        \\
        & = 1 - \sup_{\substack{g' \text{ measurable}\\0\leq g' \leq 1 \\\int g' d\S \le \int  \1_{h\neq f} d\S}} \int g' d\T
    \end{split}
\end{equation*}
Considering the particular case: $g'=\1_{h\neq f}$ one gets
\begin{equation*}
    \begin{split}
    F(1-\risk_P(h))
        & \le 1 - \int  \1_{h\neq f} d\T
        = 1 - \risk_Q(h)\\
    \end{split}
\end{equation*}
\end{proof}

The exact same bound can be expressed with the PR curve parametrization thanks to Theorem \ref{thm:duality}.
\begin{proposition}
\label{prop:pr-bound} \revision{We have the following bound:}
\begin{equation}
\label{eq:pr-bound}
    \risk_Q(h)\leq \lambda^* \risk_P(h)+(1-\alpha_{\lambda^*})
\end{equation}
with $\lambda^*=\argmax\limits_{\lambda\in\mathbf{R^+}}\lbrace\alpha_\lambda - \lambda \risk_P(h)\rbrace$
\end{proposition}

\begin{proof}
This result follows from Thm.~\ref{thm:duality}. 
\end{proof}

The first part of our upper-bound corresponds to the errors occurring in the common support of $\S$ and $\T$. There, the error rate is controlled in the source domain, and is therefore also controlled in the target domain. The amplification factor $\lambda^*$ accounts for the fact that the common mass between $\S$ and $\T$ is present in different ratios in the two domains. The second part corresponds to errors occurring in the target domain, within mass that is not present in the source domain. We do not have control over this error, and must account for it by considering the worst case where $h$ is always wrong. As a result, the only way to keep this term under control is to make some assumptions on the class of admissible functions for $h$ and the distribution of labels, that is to say the hypothesis class and the concept class.

\subsection{Discussion}
The latter form of our bound (Eq.~\eqref{eq:pr-bound}) highlights  a strong tie with Eq.~\eqref{eq:da-tv}. In particular, if $\lambda^*=1$ then, noting that $\alpha_1=1-\frac 12 \Vert P-Q\Vert_{TV}$, then the two bounds are almost identical.
The only difference resides in a factor $\frac 12$ in our favor, which stems from the fact that we explicitly leverage the non-negativity of the error. Besides in general, $1$ is not the optimal $\lambda^*$ and our bound is even tighter. 
This is where the reliance on a trade-off curve comes in handy: we obtain virtually one bound for each value of $\lambda$ and we can pick the sharpest one. Let us consider the simple example of Fig.~\ref{fig:complex} where we can read the value of $\alpha_\lambda$ on the y-axis at the location where the PR-curve meets the line of equation $\alpha = \lambda \beta$. In this example $\alpha_1\approx 0.38$ which means that $\Vert P -Q\Vert_{TV} = 2(1-\alpha_1) \approx 1.24$. Therefore the bound of Eq.~\eqref{eq:da-tv} is larger than $1$ and is hence non informative.
On the other hand, Eq.~\eqref{eq:dprd2} with $\lambda=1$ gives $\risk_Q(h)\leq \risk_P(h) + (1-\alpha_1)\approx \risk_P(h)+0.62$ which is informative as soon as $\risk_P(h) < 0.38$. This condition is easily met in concrete cases because $\risk_P(h)$ is the error in the source domain, which is under control thanks to supervision.  More importantly, we can examine how the optimal trade-off between the two terms of the error bound can yield a much sharper bound. 
In fact, given the convexity of $1-\alpha_\lambda$, the optimal $\lambda$ is characterized by the first-order critical point condition, that is $\risk_P(h)\in \partial_\lambda \alpha_\lambda$.
Note that it is not trivial to read the derivative of $\alpha_\lambda$ from the PR curve directly, but in this example this derivative is much larger than $1$ and then larger than $\risk_P(h)$. This means that the optimal $\lambda$ is far from $1$.

Besides, as discussed earlier, our bound can be easily understood in terms of shared vs separate mass between $P$ and $Q$.
Despite those advantages, one must keep in mind that it suffers from the same limitations when considering its practical estimation from finite i.i.d. samples and its pessimistic nature with regard to the actual regularity of the error set.
That being said, Lorenz and PR curves draw several similarities with tools developed in domain adaptation, which calls for further scrutiny. 
For instance, the proof of Thm~\ref{thm:pr-div} allows to express PR-curves as trade-off curves computed from weight ratios as defined in \cite{ben2012hardness}. 
Inspired by their work, it would be natural to restrict the class of ``admissible'' measurable sets in the PR-curves and get more useful bounds while retaining the notion of an optimal trade-off.
A similar step could be taken on the dual representation by restricting the class of functions in the Lorenz diagram. 
In that way, one would leverage the restrictions on the hypothesis class and get similar bounds as many of those derived from Integral Probability Metrics \citep{redko2020survey}.
Nonetheless, given the proven ability of deep neural nets to overfit random labels \cite{zhang2016understanding}, leveraging classical notions of complexity is probably not sufficient to get bounds that are representative of the  current domain adaptation problems.
It seems inevitable to leverage some kind of ‘‘implicit bias'' related to the optimization procedure.
Doing so while relying on trade-off curves is an exciting research avenue: it is on-going and will certainly yield more practical bounds for domain adaptation.

\section{Several links with $\phi$-divergences}
\label{sec:phi-div}
\revisioncomment{This section is entirely new.}

\subsection{A short digest on $\phi$-divergences}
In what follows, $\phi:\R\to\R$ is a l.s.c. convex function verifying the following assumptions:
\begin{itemize}
\item[\namedlabel{a0}{(A0)}] $\dom(\phi)\cap]-\infty[ = \emptyset$ and $]0,+\infty[\subset\dom(\phi)$ 
\item[\namedlabel{a1}{(A1)}] $\phi(1)=0$ 
\item[\namedlabel{a2}{(A2)}] $0\in\partial \phi(1)$ {and $\partial \phi(1)$ is symmetric around $0$ (ie $\phi'_{-}(1)=-\phi'_+(1)$)}
\item[\namedlabel{a3}{(A3)}] $\phi$ is strictly convex at 1
\end{itemize}
We will denote $\Phi_1$ the set of l.s.c. convex functions that follows these three assumptions.
We denote $\phi^\diamond(u)=u\phi(\frac 1u)$, the so-called Csizár dual of $\phi$. 

\begin{definition}[$\phi$-divergence]
Let $\phi\in\Phi_1$ and $\mu$ and $\nu$ two positive measures. Then,
\begin{equation}
    \label{eq:phi-div}
    D_\phi(\mu\Vert\nu):=\int \phi\left(\frac{d\mu}{d\nu}\right) d\nu +\phi(0)\nu(\mu=0) + \phi^\diamond(0)\mu(\nu=0)
\end{equation}
with $\phi(0):=\lim_{u\to 0}\phi(u)$ and $\phi^\diamond(0):=\lim_{u\to 0} u\phi(\frac 1u)$
\end{definition}

\begin{remark}
   Let us give some rationale behind the four assumptions defining $\Phi_1$. The first part of Assumption~\ref{a0} restrict the $\phi$ divergence to positive measures (or same sign measures), hence removes any subjective choice for the value of $\phi(\frac{d\mu}{d\nu})$ when $\frac{d\mu}{d\nu}<0$. 
   Taken together, the three other assumptions defining $\Phi_1$ enforce  that  $D_\phi(\mu\Vert\nu)\geq 0$ with equality iff $\mu=\nu$. Indeed, using assumptions~\ref{a1} and \ref{a2}, one sees that $\forall u\in\R$, $\phi(u)\geq\phi(1)=0$, and the last assumption implies that $\phi(u)=1$ iff $u=1$.
   Assumption~\ref{a2} is not always required in the literature when dealing with probability (because in this case, $\mu=P,\nu=Q$ are distributions and $D_\phi(P\Vert Q)$ is invariant to changes of the form  $\tilde \phi(u) = \phi(u)-\frac{\phi'_+(1)+\phi'_-(1)}{2}(u-1)$). 
   Besides, notice that $\phi\in\Phi_1$ iff $\phi^\diamond\in\Phi_1$ and the following symmetry relation holds: $D_\phi(\mu\Vert\nu)=D_{\phi^\diamond}(\nu\Vert\mu)$.
\end{remark}

\begin{remark}
   In addition to the Csizár dual, the Legendre-Fenchel dual of $\phi$ will play a key role in variational forms of $D_\phi$ \citep{nguyen2010estimating}. It is therefore interesting to comment on the impact of the assumptions $\phi\in\Phi_1$ on $\phi^*$. 
   In fact, $\phi(1)=\min_{u}\phi(u)$ means exactly that $\phi^*(0)=-\phi(1)$. Besides, since $\phi(1)=0$, it means that $\phi^*(0)=0$.
   Symmetrically, when $0\in\dom(\phi)$, let $v_0\in\partial \phi(0)$, then
   $\phi^*(v_0)=\min_v\phi^*(v)$. Besides, since $\phi$ is minimized at $u=1$, then it is decreasing wherever $u<1$ and increasing otherwise (by convexity).
   In particular, $\partial \phi(0)\subset [-\infty,0]$ and therefore the minimum of $\phi^*$ can only be reached at a point $v_0\leq 0$ (it is also possible that the infimum is not reached if $\partial \phi(0)=\{-\infty\}$ or if $0\not\in\dom(\phi)$).
   To conclude this remark, let us notice that the definition of $\phi$ on $]-\infty,0]$ is irrelevant for $D_\phi(\mu\Vert\nu)$ (indeed, since $\mu$ and $\nu$ are assumed positive, then $\frac{d\mu}{d\nu}\geq 0$).
   As a result, $D_\phi$ is also not impacted by the values taken by $\phi^*(v)$ for $v<v_0:=\sup \{v\in\argmin(\phi^*)\}$. In any case, since assumption~\ref{a0} imposes $\phi(\frac{d\mu}{d\nu}) = +\infty$ when $\frac{d\mu}{d\nu}<0$ then $\phi^*(v)=\phi^*(v_0)$ whenever $v<v_0$. 
   To simplify concrete formulae, e.g. in Table~\ref{tab:phi-div}, we will focus on $\phi^*(v)$ on the following restriction of its domain $\dom^{v_0\to}(\phi^*):= \cup_{u\geq 0} \partial\phi(u) =\dom(\phi^*)\cap [v_0,+\infty[$.  
  It includes at least $]v_0,\phi'_+(+\infty)[$, and since $\phi$ is strictly convex at $1$, $\phi'_+(\infty)>0$ (indeed $\phi'_+(\infty)\geq \phi'(1)\geq 0$ and if by the symmetry assumption of (A2) $\phi'_+(1)=-\phi'_-(1)=0$ and then strict convexity implies that $\phi'_+(\infty) >\phi'_+(1)$. Therefore $\dom(\phi^*)\supset ]v_0,0]$. 
\end{remark}

\begin{table}[]
    \centering
    \begin{tabular}{r|c|c|c|c|c}
    \hline
        Divergence & $\phi(u)$ & $\phi'(u)$ & $\dom^{v_0\to}(\phi^*)$ & ${\phi^*}'(v)$ & $\phi^*(v)$ \\
    \hline
        KL         & $u\log(u)-(u-1)$ & $\log(u)$ & $\R$ & $e^v$ & $e^v -1$ \\
        rKL        & $-\log(u)+(u-1)$ & $1-\frac 1u$ & $]-\infty,1[$ & $\frac{1}{1-v}$ & $-\log(1-v)$ \\
        JS         & $\begin{array}{cc}
             -(u + 1) \log \frac{1+u}2 &  \\
             + u \log u& 
        \end{array}$ & $\log(\frac{2u}{1+u})$ & $]-\infty,\log(2)[$ & $ \frac{e^v}{2-e^v}$ & $ -\log(2-e^v) $ \\
        $\chi^2_{\text{Pearson}}$ & $(u - 1)^2$ & $2(u-1)$ & $[-2,+\infty[$ & $ \frac{v}{2}+1$ & $ \frac{v^2}4+v $ \\
        Hellinger  & $(\sqrt u - 1)^2$ & $1-\frac 1{\sqrt{u}}$ & $]-\infty,1[$ & $ \frac{1}{(1-v)^2}$ & $ \frac v{1-v} $ \\
        TV         & $|u - 1|$ & $\sign(u-1)$ & $[-1,1]$ & $1$ & $ v $ \\
        \hline
    \end{tabular}
    \caption{A few standard $\phi$-divergences. For each choice of $\phi$ we provide the entries in the natural derivation order, namely we compute first $\phi'(u)$ then we derive the ``meaningful'' domain of $\phi^*$, that is $\dom^{v_0\to}(\phi^*)$, then ${\phi^*}'(v)$ is determined as the inverse function of $\phi'(u)$ and $\phi^*(v)$ is determined as the anti-derivative that verifies $\phi^*(0)=0$. Note that $\phi^*$ and its derivative are only given on their meaningful domain.}
    \label{tab:phi-div}
\end{table}

\subsection{Integral representation of an f-divergence using the precision-recall curve}
\label{sec:phi-f1}
Considering two distributions $P$ and $Q$, the very definition of $D_\phi(P\Vert Q)$ involves $P(Q=0)=1-\beta_0$ and $Q(P=0)=1-\alpha_\infty$ that is to say the gap between the best possible precision/recall and their actual values. 
In fact, building upon known integral representations of $\phi$-divergences \citep{liese2006divergences}, the entire divergence can be expressed as an integral of weighted precision-recall gaps.
Let us first reformulate the link hilighted in Remark~\ref{rmk:groot} between De Groot's statistical information and the precision-recall curve.

\begin{proposition}
Let $P,Q$ two distributions, $\pi\in[0,1]$ and denoting $F^1_{\pi}(P,Q):=\frac 2{\tfrac 1{\alpha_{\lambda}}+\tfrac 1{\beta_{\lambda}}}$ the $F^1$-score associated with the precision-recall pair $(\alpha_\lambda, \beta_\lambda)$ where $\lambda=\frac{\pi}{1-\pi}$. Then:
\begin{equation*}
   F^1_\pi(P,Q) = 2 B_\pi(P,Q) 
\end{equation*}
As a result, the De Groot's statistical information is related to the gaps of $F^1$ scores as follows:
\begin{equation}
    \label{eq:bpi-f1}
    \Delta B_\pi(P,Q) = \frac 12 \Delta F^1_\pi(P,Q)
\end{equation}
with $\Delta F^1_\pi(P,Q):=F^1_\pi(P,P)-F^1_\pi(P,Q)$.
\end{proposition}
\begin{proof}
Let us first recall the definition of $B_\pi(P,Q):=(\pi P\wedge(1-\pi)Q)(\measSpace)$. The proposition derives from a simple computation and the fact that $\alpha_\lambda=\lambda\beta_\lambda$:
\begin{equation*}
    \begin{split}
    F^1_{\pi} :=& \frac{2\alpha_{\frac{\pi}{1-\pi}}\beta_{\frac{\pi}{1-\pi}}}{\alpha_{\frac{\pi}{1-\pi}}+\beta_{\frac{\pi}{1-\pi}}}\\
    =& \frac{2\alpha_{\frac{\pi}{1-\pi}}}{\frac{\pi}{1-\pi}+1} =2(1-\pi)\alpha_{\frac\pi{1-\pi}}\\
    =& 2(1-\pi)[\frac\pi{1-\pi}P\wedge Q](\measSpace) = 2 B_\pi(P,Q)
    \end{split}
\end{equation*}
\end{proof}

From the previous result, one can restate the Theorem~11 of \cite{liese2006divergences} in terms of precision-recall $F^1$ scores.
\begin{corollary}
Let $P,Q$ two distributions and $\phi\in\Phi_1$. Then,
\begin{equation}
    \label{eq:phi-int} 
  D_\phi(P\Vert Q) = \frac 12 \int_0^1  \Delta F^{1}_{\pi}(P,Q)d\Gamma_\phi(\pi)  
\end{equation}
where $d\Gamma_\phi(\pi):= \frac 1\pi d\phi_+'(\frac{1-\pi}\pi)$ intrinsically represents the distribution of curvature of the convex function $\phi$.
\end{corollary}
\begin{proof}
This is a mere restatement of the result of \cite[Theorem~11]{liese2006divergences}, based on Equation~\ref{eq:bpi-f1}.
\end{proof}

\begin{remark}
Therefore, depending on the curvature distribution of $\phi$, minimizing the $\phi$-divergence between $P$ and $Q$ (as done in f-GAN \citep{nowozin2016f}) is equivalent to minimizing a weighted version of the gap between the best $F^{1}_{\pi}$ score and the actual score.
For instance, one can consider the following family of functions $\phi_a$ and their associated $\phi$-divergence, that we will refer to as Tsallis\footnote{Note that the normalizing constant is usually set as $\frac 1{a-1}$ for Tsallis divergences instead of $\frac 1{a(a-1)}$.
Yet the normalization chosen here presents several advantages : in particular the obtained family of divergence is well defined for all $a\in\R$ (instead of only $a\geq 0$) and besides the Csizár dual is given by $\phi_a^\diamond = \phi_{1-a}$. The obtained divergence was simply called $\alpha$-divergence in \cite{poczos2011estimation}.} $a$-divergence \citep{tsallis1988possible}:

\begin{equation}
    \label{eq:phi_a}
    \phi_a(u):=\frac 1{a(a-1)}(u^a-a(u-1)-1)
\end{equation}
In that case, the curvature is given by $\phi_a''(u) = u^{a-2}$ and $d\Gamma_{\phi_a}(\pi) = \frac 1{\pi^3}\phi_a''\left(\frac{1-\pi}{\pi}\right)d\pi = \frac 1{\pi^3}\left(\frac{1-\pi}\pi\right)^{a-2} d\pi$.
As a result, one gets
\begin{equation*}
    D_{\phi_a}(P\Vert Q) = \frac 12 \int_0^1 \Delta F^1_\pi(P,Q) \left(\frac{1-\pi}\pi\right)^{a} \frac 1{\pi(1-\pi)^2} d\pi
\end{equation*}
One can see that depending on the value of $a$, this objective will focus rather on regions where precision is more important $\pi\to 1$ or on those where recall is paramount $\pi\to 0$.
\end{remark}

\subsection{Divergence frontiers parameterization with $\phi$-divergences}

Following \cite{pmlr-v108-djolonga20a} with a minor correction, one can consider the following monotone transform of the Rényi divergence\footnote{Actually we had to amend slightly the transform used in \cite{pmlr-v108-djolonga20a} because of a minor flaw in their exposition as well as out of personal convenience.} and recover the previously introduced Tsallis $a$-divergence: 
$\hat D_a(\mu\Vert P):=\frac{1}{a(a-1)} \left(\exp\left((a-1)D_a(\mu\Vert P)\right)-1\right)=\frac{1}{a(a-1)}\int \left(\frac{d\mu}{dP}\right)^a -1 dP = D_{\phi_a}(\mu\Vert P)$.
Since this is a $\phi$-divergence, it might be expected through the link developed in section~\ref{sec:phi-f1} that  this divergence frontier is determined by the precision-recall curve. 
Making this statement concrete requires a bit of work though, because the divergence involved here is not directly expressed between $P$ and $Q$ but implies an auxiliary  distribution $\mu$.
The following proposition materializes the anticipated link.

\begin{proposition}
\label{prop:div-front-phi}
Let $P, Q$ two distributions. The above defined Tsallis divergence frontier can be obtained as the pairs: $(\hat D_a(\mu_\theta\Vert P), \hat D_a(\mu_\theta\Vert Q))$ with $\theta\in[0,1]$ and 
\begin{equation}
    \label{eq:mutheta}
    \frac{d\mu_{\theta,a}}{dP}=\frac{(\theta + (1-\theta)\left(\frac{dQ}{dP}\right)^{1-a} )^{\frac 1{1-a}}}{\int(\theta + (1-\theta)\left(\frac{dQ}{dP}\right)^{1-a} )^{\frac 1{1-a}}dP} 
\end{equation}
As a result, 
\begin{equation}
    \label{eq:div-asphi}
    \hat D_a(\mu_{\theta,a}\Vert P) = \frac 1{a(a-1)}\frac{\int(\theta + (1-\theta)\left(\frac{dQ}{dP}\right)^{1-a} )^{\frac a{1-a}}dP}{\left(\int(\theta + (1-\theta)\left(\frac{dQ}{dP}\right)^{1-a} )^{\frac 1{1-a}}dP\right)^a}
    \end{equation}
\end{proposition}
\begin{proof}
Note that this proposition was already presented in \cite[Proposition~1,Proposition~3]{pmlr-v108-djolonga20a}. 
We still provide a proof because the one presented in \cite{pmlr-v108-djolonga20a} lacks details.
The first point noted in \cite{pmlr-v108-djolonga20a}, is that contrary to the Rényi divergence, the Tsallis version is convex in both arguments, and in particular in $\mu$. As a result the bi-objective optimization associated with the Pareto-frontier $\partial \hat R(P,Q)$ can be solved through linear scalarization.
This means that the divergence frontier for the Tsallis divergence is parametrized as
$$\partial \hat R(P,Q) = \{(\hat D_a(\mu_{\theta,a}\Vert P),\hat D_a(\mu_{\theta,a}\Vert Q)), \theta \in [0,1]\} $$
where $$\mu_{\theta,a} =\argmin_{\mu\in\AC(P,Q)} \theta \hat D_a(\mu\Vert P)+(1-\theta)D_a(\mu\Vert Q)$$
This is a generalized centroid problem that has been first solved in \cite{amari2007integration}, as follows. 
Denoting, with a slight abuse of notation, $m_\theta=\frac{d\mu_{\theta,a}}{dP}$, the optmization problem becomes:
$$m_\theta =\argmin_{\mu = m dP/\int m dP = 1} \theta \hat D_a(\mu\Vert P)+(1-\theta)D_a(\mu\Vert Q)$$
Using a Largange multiplier, the first order optimality condition is:
\begin{equation}
    \begin{split}
        \int  m_\theta^{a-1}(\theta+(1-\theta)\left(\frac{dP}{dQ}\right)^{a-1})\delta m_\theta dP + \lambda \int \delta m_\theta dP =0
    \end{split}
\end{equation}
which yields $m_\theta^{1-a}\propto(\theta+(1-\theta)\left(\frac{dP}{dQ}\right)^{a-1})$. 
Using the normalization constraint, one gets:
$$m_\theta = \frac{d\mu_{\theta,a}}{dP}=\frac{(\theta + (1-\theta)\left(\frac{dP}{dQ}\right)^{a-1} )^{\frac 1{1-a}}}{\int(\theta + (1-\theta)\left(\frac{dP}{dQ}\right)^{a-1} )^{\frac 1{1-a}}dP}=\frac{(\theta + (1-\theta)\left(\frac{dQ}{dP}\right)^{1-a} )^{\frac 1{1-a}}}{\int(\theta + (1-\theta)\left(\frac{dQ}{dP}\right)^{1-a} )^{\frac 1{1-a}}dP}$$

As a result,
\begin{equation}
    \begin{split}
       \hat D_a(\mu_{\theta,a}\Vert P) =& \frac{1}{a(a-1)}\int \left(\frac{d\mu_{\theta,a}}{dP}\right)^a-1 dP\\
       =& \frac 1{a(a-1)}\left(\frac{\int(\theta + (1-\theta)\left(\frac{dQ}{dP}\right)^{1-a} )^{\frac a{1-a}}dP}{\left(\int(\theta + (1-\theta)\left(\frac{dQ}{dP}\right)^{1-a} )^{\frac 1{1-a}}dP\right)^a}-1\right)
    \end{split}
\end{equation}
\end{proof}

\begin{remark}
The previous parameterization of the divergence frontier was given in \cite[Proposition~3]{pmlr-v108-djolonga20a}. 
It is expressed entirely in terms of $\phi$-divergences between $P$ and $Q$. As a result, 
the divergence-frontier is completely determined by the DeGroot statistical information $\Delta B_\pi(P,Q)$ which in turns are in bijection with the Precision-recall curve between $P$ and $Q$.
From this observation, one can state that for any exponent $a$, the Rényi divergence frontier is entirely determined by the precision recall curve. 
In order words, the precision-recall curve is always more complete a description than any Rényi divergence frontier.
The same remark holds for the interpolated f-divergence curves proposed in \cite{liu2021divergence} 
which, by design, are parameterized with f-divergences between $P$ and $Q$.
That being said, the estimation from a finite sample of the precision-recall curve in a non-parametric settings might reveal more complicated than divergence frontiers. This question deserves further scrutiny, and although it is out of the scope of this article, the interesting reader can refer to \citep{rubenstein2019practical,liu2021divergence} and references therein.
\end{remark}

\begin{remark}
   The reader can verify that taking the limit case $a\to\infty$ for the parameterization $\mu_{\theta,a}$ tends to a single point of the precision recall curve, precisely
   $\mu_{\theta,\infty}=\frac{P\wedge Q}{(P\wedge Q)(\measSpace)}$. 
   In order to recover the full curve, one should use a parameterization that depends on the Rényi exponent $a$, by setting:
   $(\theta_a, 1-\theta_a)\propto (\pi^{1-a}, (1-\pi)^{1-a})$ where $\pi \in [0,1]$ is the new parameter. In that case,
   $$\frac{d\mu_{\theta_a,a}}{dP} = \frac{(\pi^{1-a} + \left((1-\pi)\frac{dQ}{dP}\right)^{1-a} )^{\frac 1{1-a}}}{\int(\pi^{1-a} + \left((1-\pi)\frac{dQ}{dP}\right)^{1-a} )^{\frac 1{1-a}}dP}$$
   Then $\mu_{\theta_a, a}\to \frac{\pi P\wedge (1-\pi)Q}{(\pi P\wedge (1-\pi)Q)(\measSpace)}$ as $a\to\infty$.
   With such a parameterization, the numerator and denominator  in the expression of  $\hat D_a(\mu_{\theta_a,a}\Vert P)$ are reminiscent of Arimoto divergences (see e.g. \citep{liese2006divergences}).
\end{remark}

\subsection{Generalized Lorenz diagrams associated with an f-divergence}
In \cite{nguyen2010estimating}, they show that given a class of bounded measurable function $\mathcal F$,
\begin{equation}
\begin{split}
D_\phi(P \Vert Q) =& \int \phi(\frac{dP}{dQ})dQ =\int \sup_{v\in\R} v\frac{dP}{dQ} - \phi^*(v) dQ \\
\geq& \sup_{f\in \mathcal F} \int f dP - \int \phi^*(f) dQ
\end{split}
\end{equation}
with equality iff $\partial \phi(\frac{dP}{dQ})\cap \mathcal F\neq\emptyset$, which means that there exists $f\in \mathcal F$ s.t $f\in\partial\phi(\frac{dP}{dQ})$ (or equivalently thanks to duality $\frac{dP}{dQ} \in \partial\phi^*(f)$). Note that with standard conventions, $0\in\partial \phi(1)$ and since convexity implies that subdifferential of $\phi$ are increasing, then the condition  $f\in \partial\phi(\frac{dP}{dQ})$ imposes that $f\leq 0$ when $\frac{dP}{dQ}\leq 1$ and $f\geq 0$ otherwise.

This variational form can be used to create a kind of Lorenz diagram, associated with $D_\phi$, as the set of pairs $\{(\int \phi^*(f) dP, \int fdQ)\}$. Then, the upper frontier of this region is characterizing the closeness of $P$ and $Q$, and in particular $D_\phi(Q\Vert P)$ is the maximal vertical distance between the curve and the diagonal. For technical reasons (mainly ensuring a convex diagram), the definition of the Lorenz diagram is slightly more convoluted (see Figure~\ref{fig:generalized-lorenz} for an illustration of its construction\footnote{Note that this construction can be adapted to the tighter variational formulation of $D_\Phi$ proposed in \cite{ruderman2012tighter} and \cite{agrawal2020optimal}}).

\begin{figure}[ht]
    \centering
    \includegraphics[width=0.8\textwidth]{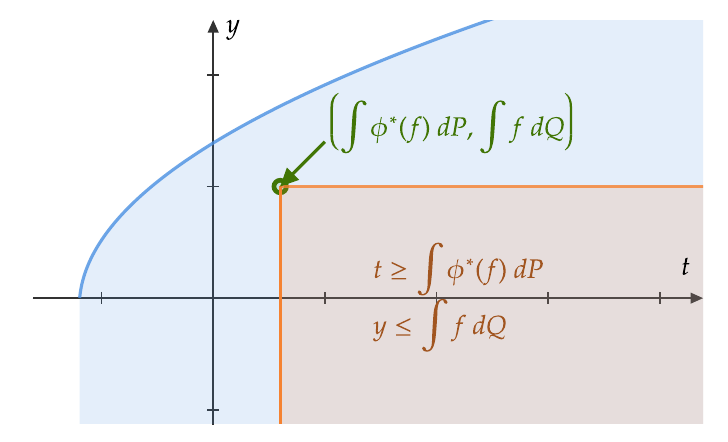}
    \caption{Generalized Lorenz diagram construction.}
    \label{fig:generalized-lorenz}
\end{figure}

\begin{definition}[extended Lorenz diagram] Let $\phi\in\Phi_1$ and $P,Q$ two distributions. The $\phi$-Lorenz diagram between $P$ and $Q$ is defined as:
\begin{equation}
   \LD_\phi(P,Q)=\left\lbrace (t,y) \in \R^2 / \exists f:\measSpace\to\R, t\geq\int \phi^*(f) dP, y\leq \int f dQ    \right\rbrace,
\end{equation}
The extended Lorenz curve is defined as its upper envelop:
\begin{equation}
    \bar F_\phi(t):=\sup \{y / (t,y)\in\LD_\phi(P,Q)\}
\end{equation}
\end{definition}

\begin{remark}
An attempt to derive a generalization bound for domain adaptation based on the variational form of $\phi$-divergences has been recently published in \cite{acuna2021f}.
They indeed proposed a bound similar to Equation~\eqref{eq:da-tv} in  \cite[Lemma~1]{acuna2021f} for a cost function $\ell(\hat y; y)\in \dom(\phi^*)$. 
Denoting,
\begin{equation}
    \risk^\ell_Q(h):=\int \ell(h(x); y) dQ(x,y)
\end{equation}
their bound reads as,
\begin{equation}
    \risk^{\phi^*\circ\ell}_Q(h):=\int \phi^*(\ell(h(x); y)) dQ(x,y)
    \leq \risk^\ell_P(h) + D_\phi(P\Vert Q)
\end{equation}
and noting that under the standard assumption $\phi(1)=0$ then $\phi^*(v)\geq v$, the previous bound implies that:
\begin{equation}
    \label{eq:da-phi-acuna}
    \risk^{\ell}_Q(h)\leq \risk^{\phi^*\circ\ell}_Q(h)
    \leq \risk^\ell_P(h) + D_\phi(P\Vert Q)
\end{equation}
Unfortunately, this bound is erroneous in general as can be seen by considering the family of $\phi$-divergence associated with $\phi_s = s \phi$ where $s>0$. Indeed if the bound was true, since $D_{\phi_s}(P\Vert Q) = s D_\phi(P\Vert Q)$, letting $s\to 0$, the bound would imply that $\risk^\ell_Q(h)\leq \risk^\ell_P(h)$, which is obviously wrong in general\footnote{The origin of the flaw is a fallacious switching between absolute values and a suppremum: $D_\phi(P\Vert Q)=\sup_{f\in \dom(\phi^*)} \int f dP-\int \phi^*(f) dQ = |\sup_{f\in \dom(\phi^*)} \int f dP-\int \phi^*(f) dQ|\leq\sup_{f\in \dom(\phi^*)} |\int f dP-\int \phi^*(f) dQ| $ : the last inequality is strict in general (and in fact the RHS is often infinite), while \cite{acuna2021f} used an equality.}.
\end{remark}

Similarly to Proposition~\ref{prop:error_bounds_Lorenz}, one can actually amend this bound, by using the Lorenz diagram associated to the $\phi$-divergence.
\begin{proposition}\label{prop:error_bounds_Lorenz_phi}
Let $\phi\in\Phi_1$, $\ell(\hat y; y)\in\dom(\phi^*)$.
The Lorenz curve provides the following lower bound for domain adaptation:
\begin{equation}
    \label{eq:da-phi-tradeoff}
    \risk^\ell_Q(h):=\int \ell(h(x); y) dQ(x,y) \leq  \risk^{\phi^*\circ\ell}_{\lambda P}(h) + D_\phi(Q\Vert \lambda P)
\end{equation}
\end{proposition}

\begin{proof}
Let us first produce a bound for the upper envelop of the Lorenz diagram. For $t\geq \in\dom(\phi^*)$ 
\begin{equation}
    \begin{split}
        \bar F_\phi(t) =& \sup\{y\in\R / \exists f\in\dom(\phi^*), y\leq \int f dQ, t\geq \int \phi^*(f)dP\} \\
        =&\sup_{f/\int \phi^*(f) dP \leq t} \int f dQ = \sup_{f} \inf_{\lambda\geq 0} \int f dQ -\lambda(\int \phi^*(f) dP - t)\\
        \leq& \inf_{\lambda\geq 0} \sup_{f} \int f dQ -\lambda(\int \phi^*(f) dP - t) \\
        =&  \inf_{\lambda\geq 0} \lambda t + \sup_{f} \int f dQ - \int \phi^*(f)d(\lambda P)  \\
       =& \inf_{\lambda\geq 0} \lambda t + D_\phi( Q\Vert \lambda P)
    \end{split}
\end{equation}
where the inequality in the third line is trivial\footnote{It can be interpreted in terms of primal-dual gap and since when $t> t_{\min}:=\min \phi^*(v)= -\phi(0)$ Slater's condition are easily verified (as the constant function $f=\argmin_{v} \phi^*(v)$ is in the relative interior of the constraints) it is in fact an equality most of the time.}.
Note also that in the last identity, one needs to extend the variational form of $D_\phi$ from \cite{nguyen2010estimating} when applied to $Q$ and $\lambda P$ (which is a positive measure but not a probability distribution in general).
The demonstration of \cite{nguyen2010estimating} extends easily to this situation.
As a result, 
\begin{equation}
\begin{split}
    \risk^\ell_Q(h):=&\int \ell(h(x); y) dQ(x,y) \\
    \leq& \bar F_\phi(\risk^{\phi^*\circ\ell}_P(h)) \\
    \leq& \inf_{\lambda>0} \lambda \risk^{\phi^*\circ\ell}_P(h) + D_\phi(Q\Vert \lambda P)
\end{split}
\end{equation}
\end{proof}

\begin{remark}
As for the standard Lorenz diagram, this  bound allows to trade the error made in the source domain for a better coverage between the target  distribution $Q$ and the non-normalized source distribution $\lambda P$. When using the suboptimal choice $\lambda=1$, one recovers a bound similar to the erroneous Equation~\eqref{eq:da-phi-acuna} derived from \cite{acuna2021f}:
\begin{equation}
    \label{eq:da-phi}
    \risk^{\ell}_Q(h)\leq  \risk^{\phi^*\circ\ell}_P(h) + D_\phi(Q\Vert P)
\end{equation}
In this corrected bound, the source loss $\ell$ is replaced by $\phi^*\circ\ell$ which is always larger.
{ That being said, in typical domain adaptation scenarios, $\ell$ would be small in average in the source domain (distribution $P$). Since $\ell$ is also positive in concrete cases, being small in average means being small with high probability under $P$ (for instance using the Markov's inequality). Besides, since $\phi\in\Phi_1$, $0\in \partial \phi(1)$ which by duality implies that $1\in\partial \phi^*(0)$ and in turn shows that $\phi^*(0)=-\phi(1)=0$ (provided $0\in\dom(\phi^*)$ of course). Moreover, $\phi^*$ is often differentiable at $0$, so its Taylor expansion is $\phi^*(v) = v + o(v)$, which implies essentially that $\risk^{\phi^*\circ\ell}_P(h)$ and $\risk^{\ell}_P(h)$ are similar (again if $\risk^{\ell}_P(h)$ is small). To make such a statement more accurate, let us rely on a Taylor expansion with integral remainder. If ${\phi^*}'(0)=1$ then one has the following Taylor expansion (see \cite[Theorem~1]{liese2006divergences}):
$$\phi^*(v) = v + \int_0^v (v-s)d{\phi^*_+}'(s)$$
In particular, if the curvature of $\phi^*(v)$ is bounded : $\forall v\in [0,1], {\phi^*}''(v)\in [\underline \kappa, \bar\kappa]$, then one has the following estimates:
$$\frac 12 \underline \kappa \sigma^2 \leq \risk^{\phi^*\circ\ell}_P(h)- \risk^{\ell}_P(h)\leq \frac 12 \bar\kappa \sigma^2$$
where $\sigma^2:=\int \ell(h(x),y)^2 dP$ is usually quite small in concrete scenarios.
Note that the previous reasoning is valid only if the curvature of $\phi^*$ is controlled. For instance, the counter-example based on replacing $\phi$ by a rescaled version $\phi_s=s\phi$ would break that assumption if $s\to 0$. Indeed, $\phi_s^*(v)=s\phi^*(\frac v s)$, and then ${\phi_s^*}''(v)=\frac 1 s {\phi^*}''(\frac v s)$.
}
\end{remark}

As advocated in the domain adaptation literature \cite{ben2010theory}, to be useful, a generalization bound should :
\begin{itemize}
    \item rely on notions that can be estimated from finite samples (this is not specific to domain adaptation)
    \item minimize their dependance on the distribution of labels in the target domain $Q_{Y|X}$  (especially if one aims at deriving a learning algorithm from the bound)
\end{itemize}
To do so we will rely on triangle inequalities.
\begin{definition}
Let $\ell:\mathcal Y\times\mathcal Y\to \R$ a loss function. We way that
\begin{itemize}
    \item $\ell\in\tiA$ iff $\forall y_a,y_b, y_c\in \mathcal Y, \ell(y_a; y_c)\leq \ell(y_a; y_b)+\ell(y_b; y_c)$.
    \item $\ell\in\tiB$ iff $\forall y_a,y_b, y_c\in \mathcal Y, \ell(y_a; y_c)\leq \ell(y_a; y_b)+\ell(y_c; y_b)$.
\end{itemize}
\end{definition}

\begin{remark}
Both triangle inequalities are valid for the margin loss of \cite{zhang2019bridging}:
\begin{equation}
\label{eq:margin-loss}
\ell(\hat y; y) = \left[1-\frac 1\rho[\delta(\hat y; y)]_+\right]_+=\left\{
\begin{array}{ll}
    1& \text{ if } \argmax(\hat y)\neq \argmax(y)\\
    \max(0,1-\frac{\delta(\hat y; y)}\rho) & \text{ otherwise}
\end{array}
\right.
\end{equation}
where $[x]_+:=\max(0,x)$ and $\delta(\hat y; y)=\frac 12 \min_{k\neq k_y} \hat y_{k_y} -\hat y_k$ is a multi-class margin between the score $\hat y_{k_y}$ attributed to the largest component of $y$ ($k_y=\argmax_k y_k$) and the best challenger score $\max_{k\neq k_y} \hat y_k$.
\end{remark}

The following lemma is representative of the typical ways to turn a bound such as Equation~\eqref{eq:da-phi-tradeoff} into one that suits the practical purposes mentionned earlier.
\begin{lemma}
\label{lemma:triangle}
Let $h, h'\in\mathcal H$ two hypotheses, $\ell$ a loss function and $\mu$ a positive measure\footnote{$\mu$ could be $P$ or $Q$ or even $\lambda P$.} over pairs $(x,y)$. Then, noting $\mu_X$ the marginal of $\mu$ w.r.t $x$, 
\begin{itemize}
    \item if $\ell\in\tiA$
\begin{equation}
    \risk_\mu^\ell(h) \leq \risk_\mu^\ell(h') + \risk_{\mu_X}^\ell(h;h')
\end{equation}
where $\risk_{\mu_X}^\ell(h;h')=\int \ell(h(x);h'(x)) d\mu_X$.
\item
Similarly, if $\ell\in\tiB$
\begin{equation}
    \risk_{\mu_X}^\ell(h;h')\leq \risk_\mu^\ell(h) + \risk_\mu^{\ell}(h') 
\end{equation}
\end{itemize}
\end{lemma}
\begin{proof}
Indeed, if $\ell\in\tiA$
\begin{equation}
\begin{split}
    \risk_\mu^\ell(h) =& \int \ell(h(x);y) d\mu(x,y) \leq \int \ell(h(x);h'(x)) +\ell(h'(x);y) d\mu(x,y)\\
    =& \risk_\mu^\ell(h') + \risk_{\mu_X}^\ell(h;h')
\end{split}
\end{equation}
Similarly, if $\ell\in\tiB$
\begin{equation}
\begin{split}
     \risk_{\mu_X}^\ell(h;h')=& \int \ell(h(x);h'(x)) d\mu_X(x) = \int \ell(h(x);h'(x)) d\mu(x,y) \\
     \leq&  \int \ell(h(x);y) +\ell(h'(x); y) d\mu(x,y)\\
    =& \risk_\mu^\ell(h) + \risk_{\mu}^{\ell}(h')
\end{split}
\end{equation}
\end{proof}

Following \cite{acuna2021f}, and taking into account our correction to their bound based on $\phi$ divergences, we propose the following notion of discrepancy.
\begin{definition}[$D_{h,\mathcal H}^{\phi,\ell}(P_X, Q_X)$]
\label{def:phi-disc}
Let $P_X, Q_X$ two distributions over $x$ and $\lambda>0$. Let $\phi\in\Phi_1$, and $\ell$ a loss function.
Let also $\mathcal H$ a class of hypotheses and $h\in\mathcal H$. Then we define,
\begin{equation}
    \label{eq:phi-disc}
    D_{h,\mathcal H}^{\phi,\ell}(Q_X\Vert \lambda P_X):=\sup_{h'\in \mathcal H} \risk_{Q_X}^\ell(h; h') -\risk_{\lambda P_X}^{\phi^*\circ\ell}(h;h')
\end{equation}
where by extension $\risk_{\lambda P_X}^{\phi^*\circ\ell}(h;h'):=\int \phi^*(\ell(h; h')) d\lambda P_X = \lambda \risk_{P_X}^{\phi^*\circ\ell}(h;h')$.
\end{definition}

\begin{corollary}
\label{prop:phi-da-bound}
Let $\phi\in\Phi_1$. Let also $\ell\in\tiA$ and $\ell_\phi\in\tiB$ such that $\phi^*\circ\ell \leq \ell_\phi$.
Given an hypothesis $h\in\mathcal H$, we have the following generalization bound:
\begin{equation}
    \label{eq:phi-da-genbound}
    \risk_Q^\ell(h) \leq \inf_{\lambda >0} \risk_{\lambda P}^{\ell_\phi}(h) + D_{h,\mathcal H}^{\phi,\ell}(Q_X\Vert \lambda P_X) +\gamma_\lambda^*
\end{equation}
where $\gamma_\lambda^*:=\inf_{h'\in H}\risk_Q^\ell(h')+\risk_{\lambda P}^{\ell_\phi}(h')$ quantifies the {\em adaptibility} of the task between domains $P$ and $Q$.
\end{corollary}
\begin{proof}
From Lemma~\ref{lemma:triangle} and Definition~\ref{def:phi-disc}, for any $h'\in\mathcal H$
\begin{equation*}
\begin{split}
    \risk_Q^\ell(h) \leq&  \risk_Q^\ell(h') +  \underbrace{\risk_{Q_X}^\ell(h;h')}_{\mathrlap{\leq \risk_{\lambda P_X}^{\phi^*\circ\ell}(h;h')+ D_{h,\mathcal H}^{\phi,\ell}(Q_X\Vert \lambda P_X)}}\\
    \leq& \risk_Q^\ell(h') +\underbrace{\risk_{\lambda P_X}^{\phi^*\circ\ell}(h;h')}_{\mathrlap{\leq\risk_{\lambda P_X}^{\ell_\phi}(h;h')\leq \risk_{\lambda P}^{\ell_\phi}(h) + \risk_{\lambda P}^{\ell_\phi}(h')}}+ D_{h,\mathcal H}^{\phi,\ell}(Q_X\Vert \lambda P_X) \\
    \leq& \risk_P^{\ell_\phi}(h) + D_{h,\mathcal H}^{\phi,\ell}(Q_X\Vert \lambda P_X) + (\risk_Q^\ell(h') + \risk_{\lambda P}^{\ell_\phi}(h'))
\end{split}
\end{equation*}
Given that $h'$ was chosen arbitrarily in $\mathcal H$ to begin with, the last term can be replaced by its infimum over $h'\in\mathcal H$, yielding $\gamma_\lambda^*$.
\end{proof}

\begin{remark}
The careful reader could legitimately wonder why we introduce an alternate loss $\ell_\phi$ instead of using $\phi^*\circ\ell$ which appears as a natural choice. This is done to provide more flexibility on the choice of $\ell$ and $\phi$ as, when $\phi\in\Phi_1$ and $\ell\in\tiA$, then in general $\phi^*\circ\ell\not\in \tiB$.
This can be verified for example for the KL-divergence and the margin loss defined in Equation~\eqref{eq:margin-loss}. On the contrary, since in this case $\ell(\hat y;y)\in [0,1]$, then one can see that $\phi^*\circ \ell \leq \ell_\phi:=\phi^*(1)\ell$, which meets the hypothesis of the theorem.
\end{remark}

\begin{remark}
A part from the trade-off controlled by $\lambda$, the previous bound is representative of typical reasoning in domain adaptation works following the approach of \cite{ben2010theory}, as for example \cite{ganin2016domain,zhang2019bridging,acuna2021f}. Let us simplify the following argument by setting $\lambda=1$ and note $\gamma^*=\gamma^*_1$ the adaptability term. This kind of bounds are in fact slightly misleading. Indeed the common algorithmic use of such bounds is the following.
In the bound (and the similar ones in the literature), all terms can be efficiently evaluated from finite samples apart from $\gamma^*$. Therefore, one can 
consider an adversarial training scheme, where $h$ and $h'$ are two concurrent networks playing the following game:
\begin{equation*}
    \min_{h\in\mathcal H}\max_{h'\in\mathcal H} \risk_P^{\ell_\phi}(h) + \risk_{Q_X}^\ell(h; h') -\risk_{P_X}^{\phi^*\circ\ell}(h;h')
\end{equation*}
This game does not lead anything useful in practical deep learning adaptation problems because usually the distributions $P_X$ and $Q_X$ are easily discriminated by the adversary $h'$ whatever the choice of $h$. 
To amend this failure, previous works actually replace the game by the following one:
\begin{equation*}
    \min_{g\in\mathcal G,h\in\mathcal H}\max_{h'\in\mathcal H} \risk_P^{\ell_\phi}(h\circ g) + \risk_{Q_X}^\ell(h\circ g; h'\circ g) -\risk_{P_X}^{\phi^*\circ\ell}(h\circ g;h'\circ g)
\end{equation*}
In other words, the two adversaries share a common feature extractor $g$ which actually plays in collaboration with $h$ and thus against $h'$. 
The rationale behind this design choice, is that one needs the ability to embed the distributions $P_X$ and $Q_X$ in a shared latent space $\mathcal Z$, where their discrimination becomes more challenging. 
Intuitively this embedding comes at a cost, the distributions  $P_{Y|Z}$ and $Q_{Y|Z}$ can become:
\begin{itemize}
    \item more stochastic than their original counterpart $P_{Y|X}$ and $Q_{Y|X}$
    \item worse, they can become more inconsistent in the sense of the adaptability constant $\gamma^*$
\end{itemize}
Although this aspect was considered in \cite{johansson2019support,zhao2019learning,siry2021inductive},
to the best of our knowledge the impact of this modification was not analyzed carefully in terms of generalization bounds. 
In fact, to do so, one needs to amend the generalization bound as follows: $\forall g\in\mathcal G, \forall h\in\mathcal H$
\begin{equation}
    \risk_Q^\ell(h\circ g) \leq \risk_P^{\ell_\phi}(h\circ g) + D_{h\circ g,\mathcal H\circ g}^{\phi,\ell}(Q_X,P_X) +\gamma_g^*
\end{equation}
where we have noted $\mathcal H\circ g=\{h\circ g/ h\in \mathcal H\}$ and $\gamma_g^*=\inf_{h'\in\mathcal H}\risk_Q^\ell(h'\circ g)+\risk_P^{\ell_\phi}(h'\circ g) $.
The main issue is that $\gamma_g^*$, besides not being computable in practice, depends on the specific choice of the feature extractor $g$.
It is not anymore an intrinsic measure of the adaptability of the transfer. 
This issue can be amplified by the fact that $g$ is playing with $h$ and as a result, it may tend to be biased towards performing a correct embedding of high probability areas of the source distribution and performing in an uncontrolled way on the target distribution high likelihood areas.
A similar argument could of course be held against the trade-off parameter $\lambda$, but there is one subtle difference: it merely impacts the weight granted to the source distribution $P$. It is therefore easier to appraise its impact as well as to control it through regularization (e.g. favoring a smaller $\lambda$).
Note that our $\lambda$ variable acts in a similar fashion to the margin parameter\footnote{Beware, in \cite{zhang2019bridging}, this parameter is denoted by $\gamma$ and the adaptability term by $\lambda$ which could bring confusion.} of \cite{zhang2019bridging} in the sense that it impacts the weight of the source risk. On the other hand, the margin parameter of \cite{zhang2019bridging} is introduced artificially and does not acts on the discrepancy term.
\end{remark}

\subsubsection{Practical considerations}

The result presented in Corollary~\ref{prop:phi-da-bound} leaves quite some room for the choice of $\phi$, $\ell$ and $\ell_\phi$. 
We believe that the combined $\phi,\ell,\ell_\phi$ configuration should follow (or remain close to) the following design principles.
\begin{itemize}
    \item[\namedlabel{wanted1}{(P1)}] $\ell\in\tiA$, $\ell_\phi\in\tiB$ and $\phi^*\circ\ell\leq\ell_\phi$ (to ensure that our bound is valid)
    \item[\namedlabel{wanted2}{(P2)}] $\ell$ should be Bayes consistent (see \cite{steinwart2007compare,tewari2007consistency})
    \item[\namedlabel{wanted3}{(P3)}] $\ell_\phi$ is convex and bounded below (useful from the optimization perspective)
\end{itemize}

For example, if one follows in the steps of \cite{zhang2019bridging} and choose for $\ell$ the margin loss from Equation~\eqref{eq:margin-loss} (with e.g. $\rho=1$), since $\range(\ell)=[0,1]$, using $\ell_\phi=\phi^*(1)\ell$, then \ref{wanted1} is valid.
In addition, one can show that  \ref{wanted2} is checked in the binary classification case, because then $\ell(\hat y;y) = (1-(\hat y y)_+)_+$ is a piecewise linear function of the margin $\hat y y$ and one can compute its $\psi$-transform as defined in \cite{bartlett2006convexity} and show\footnote{In fact, using the notations of \cite{bartlett2006convexity}, one can show that $H_\ell(\eta)=\min(1-\eta,\eta)$ and $H_\ell^-(\eta)=\max(1-\eta,\eta)$. This yields $\tilde\psi_\ell(\eta) = H_\ell^- -H_\ell = |2\eta-1|>0 \forall \eta\neq\frac 12$, yielding the Bayes consistency.} that $\psi_\ell(\epsilon)=\epsilon$ which being invertible ensures that $\ell$ is Bayes consistent. 
Unfortunately, the extension to multi-category classification breaks the Bayes consistency (see \cite[example~1]{tewari2007consistency} with respect to the Crammer and Singer methodology). That being said, $\ell$ is an upper-bound of the $0-1$ loss\footnote{More precisely $\ell(\hat y; y)\geq \1_{\argmax(\hat y_k)\neq\argmax(y_k)}$}, and one still enjoy the fact that controlling the $\ell$-risk ensures a small  $0-1$-risk.
Last, $\ell_\phi$ is bounded below but not convex. That is not a real issue given that one can replace it by a  larger convex loss such as the hinge loss.

In contrast to the previous list of guiding principles, \cite{acuna2021f} merely requires that $\range(\ell)\subset [0,1]\subset \dom(\phi^*)$. 
One may fail to see how this assumption is relevant, but it surely comes with some drawbacks. Mainly, $[0,1]\cap\dom(\phi^*)$ always misses important parts of the domain of $\phi^*$.  To illustrate this point, let us consider \cite[Proposition~1]{acuna2021f}, where to deepen the link between the adversarial game and the $\phi$-divergence, it is required that $\exists h'$ s.t. $\forall x,\ell(h(x), h'(x))\in \partial \phi(\frac{dP}{dQ}(x))$. Meeting this condition is infeasible if $\range(\ell)\subset [0,1]$ because $\partial \phi(\frac{dP}{dQ}(x))\subset ]-\infty,0]$ whenever $\frac{dP}{dQ}(x)<1$. In other words, the assumption is only met when $P\geq Q$ which actually implies $P=Q$ and is of no interest since it corresponds to the absence of domain shift.
Note however that despite our general principles, the only practical settings that we consider, that is to say the margin loss, suffers the same limitation since in this particular configuration, $\range(\ell)\subset [0,1]$. 
We foresee, that this issue can be solved by refining our definition of the generalized Lorenz diagrams, using the tighter variational representation proposed in \cite{ruderman2012tighter} and \cite{agrawal2020optimal}. This representation takes the following form:
\begin{equation*}
    D_\phi(Q\vert \lambda P) =\sup_{f} \sup_{\rho\in\R} \int f dQ - \int (\phi^*(f+\rho) -\rho) d(\lambda P) 
\end{equation*}
As can be seen, in this formulation the scalar $\rho\in\R$ allows to align the range of $f+\rho$ and the domain of $\phi^*$ that is to say the subdifferentials of $\phi$. Besides, in the advent of deriving an adversarial algorithm from this refined variant, the variable $\rho$ would naturally play a balancing role with respect to the variable $\lambda$. This extension has already been adopted for GANs \citep{terjek2021moreau} which suggests that it is appropriate for domain adaptation adversarial approaches. It is however left for future work.

\section{Conclusion}

In this work, we have studied the interconnections among several trade-off curves designed to evaluate the similarity between two probability distributions, namely precision-recall curves,
divergence frontiers, ROC and Lorenz curves.
If one connection was known to the authors of divergence frontiers, others appear to have eluded even the authors of the implied notions.
This is particularly striking for Lorenz and ROC curves which differ by mere symmetries. 
The interrelation between precision-recall and Lorenz curves is less direct, as it involves convex duality. 
That being said, it remains that the two notions are theoretically equivalent, and can be computed in practice from one another.
We hope that the exposed link will foster new research avenues for evaluation curves. 
To begin with, while the theoretical equivalence of Lorenz and PR curves has been demonstrated, the question of their empirical estimation has yet to be examined. 
For instance, investigations on potentially consistent estimators need consideration, especially in the non parametric case (although this will require refining the definitions to render the estimation from finite sample feasible). 
In particular, exposing rates of convergence of the estimators should be a worthwhile endeavor. 
Similar analysis has already been carried out for scalar metrics \cite{rubenstein2019practical,sriperumbudur2009integral}. 
We foresee that exploring links with divergences or integral probability metrics shall help in pursuing this undertaking for PR and Lorenz curves.
Similar research avenues are very plausible to deepen the link we have discussed between trade-off curves and practical bounds for domain adaptation.
\revision{Instead of considering links with IPM, we have considered the case of $\phi$-divergence which have provided a fertile ground to extend the already exposed links. First, we have demonstrated that a general link exists between Rényi divergence frontiers of arbitrary exponent $a$ and PR curves. 
More notably, we have extended the notion of Lorenz diagram, and built upon this notion to derive a novel generalization bound for domain adaptation.}

\newpage
\appendix
\section{Proof of lemma~\ref{prop:sup-ratio}}
\label{app:sup-ratio}
\begin{proof}
The proof is technical and can be skipped on first reading.
The second result is a simple corollary of the first since when $\mu$ is not absolutely continuous w.r.t $\nu$ then the identity is trivial: $\infty=\infty$. Let us demonstrate the first point, that is to say when $\mu\ll\nu$.
We shall proceed by proving two opposite inequalities, starting from the following.
\begin{equation*}
         \sup_{A\in\mathcal A} \tfrac{\mu(A)}{\nu(A)} \leq \esssup_{d\mu} \tfrac{d\mu}{d\nu}.
\end{equation*}
Indeed, let $A\in\mathcal A$, then, 
\begin{equation*}
\begin{split}
    \tfrac{\mu(A)}{\nu(A)} 
    & = \frac{\int_A \tfrac{d\mu}{d\nu} d\nu}{\nu(A)} 
    = \frac{\int 1_{A} \tfrac{d\mu}{d\nu} d\nu}{\nu(A)}\\
    & \leq \esssup_{d\nu} \tfrac{d\mu}{d\nu} \tfrac{\nu(A)}{\nu(A)} = \esssup_{d\nu} \tfrac{d\mu}{d\nu}\\
\end{split}
\end{equation*}
Note that the essential sup is w.r.t $\nu$ instead of $\mu$. Let us then show that $\esssup_{d\nu} \tfrac{d\mu}{d\nu}\leq \esssup_{d\mu} \tfrac{d\mu}{d\nu}$.
To do so we need to show that any upper-bound $M$ of $\tfrac{d\mu}{d\nu}$ \pp{$\mu$} is also an upper-bound \pp{$\nu$}. Let $M\geq \tfrac{d\mu}{d\nu}$ \pp{$\mu$} be such an upperbound and let $N$ the associated $\mu$-nullset (where $M$ may be lesser than $\tfrac{d\mu}{d\nu}$). If $\nu(N)=0$ it settles it (as $M$ is henceforth an upper-bound \pp{$\nu$}).  Otherwise, $N$ is such that $\mu(N)=0$ but $\nu(N)>0$ and then necessarily $\tfrac{d\mu}{d\nu}\1_N=0$ \pp{$\nu$} (or else it would contradict $\mu(N)=0$). As a result, 
\begin{equation*}
    \tfrac{d\mu}{d\nu} \stackrel{\pp{\nu}}{=} \tfrac{d\mu}{d\nu} \1_{\measSpace\setminus N} 
    \leq M \1_{\measSpace\setminus N} \leq M \pp{\nu}
\end{equation*}
Which again settles the fact that $M$ is also an upperbound \pp{$\nu$}.
Therefore we obtain that $\esssup_{d\nu} \tfrac{d\mu}{d\nu}\leq \esssup_{d\mu} \tfrac{d\mu}{d\nu}$ and
    $\tfrac{\mu(A)}{\nu(A)} \leq \esssup_{d\mu} \tfrac{d\mu}{d\nu}$.

For the reverse inequality, we only need to show that $\tfrac{d\mu}{d\nu}\leq \sup_{A\in \mathcal A}\tfrac{\mu(A)}{\nu(A)} \pp{\mu}$. Indeed, let
$C:=\sup_{A\in \mathcal A}\tfrac{\mu(A)}{\nu(A)}$.
If $C=\infty$, the inequality is trivial. 
Let us then suppose that $C<\infty$. 
Let $E:=\lbrace\omega\in\measSpace / \tfrac{d\mu}{d\nu} > C\rbrace$ and let us show that $\mu(E)=0$. One can rewrite $E=\cup_{n\in\N}E_n$, with
$E_n:=\lbrace\omega\in\measSpace / \tfrac{d\mu}{d\nu} \geq C+\tfrac 1n\rbrace$, and it suffices to show that $\mu(E_n)=0$.
We have that,
\begin{equation*}
\begin{split}
    C\nu(E_n)\geq&  \mu(E_n)= \int_{E_n}\tfrac{d\mu}{d\nu} d\nu \geq (C+\tfrac 1n)\int_{E_n}d\nu\\
    \geq& (C+\tfrac 1n) \nu(E_n)
\end{split}
\end{equation*}
For this not to be absurd, it is necessary that $\nu(E_n)=0$ and hence $\mu(E_n)=0$ as well.
\end{proof}

\section{Proof of Lemma~\ref{cor:optimality}}
\label{app:optimality}

We will use the following result from Theorem~5 in \cite{simon2019revisiting} \revision{which stipulates that the precision  $\alpha_\lambda$ is optimal in the following sense}:

\begin{theorem}
\label{thm:icml-thm5}
Let $P,Q$ two distributions from $\probaSet$. Then 
    $\forall \lambda \in \overline{\R^+},$
    \begin{equation}
    \begin{split}
        \alpha_\lambda =& \min_{A\in\mathcal A} \lambda (1-P(A)) + Q(A)\\
    \end{split}
    \end{equation}
\end{theorem}

From this theorem the lemma is an easy corollary :
\begin{proof}
It suffices to show that for all measurable functions $0\leq f \leq 1$, 
$$\alpha_\lambda \leq\lambda (1-\int f dP) + \int f dQ$$ 
According to Theorem~\ref{thm:icml-thm5}, this inequality holds for indicator functions.
Besides it is stable under convex combinations and $L^\infty$ limits.
As a result,  we can extend the inequality first to convex combinations of indicators which is to say to all simple functions ranging in $[0,1]$. 
Note that even though at first glance, simple functions ranging in $[0,1]$ take the form of sub-convex combinations of indicators, one can leverage $\1_\emptyset=0$ to express them as convex combinations.
Then, taking the limit w.r.t to $L^\infty$ convergence and using standard density results we can further extend the inequality to
$L^\infty$ functions ranging in $[0,1]$.
\end{proof}

\vskip 0.2in
\bibliography{references}

\begin{thebibliography}{50}
\providecommand{\natexlab}[1]{#1}
\providecommand{\url}[1]{\texttt{#1}}
\expandafter\ifx\csname urlstyle\endcsname\relax
  \providecommand{\doi}[1]{doi: #1}\else
  \providecommand{\doi}{doi: \begingroup \urlstyle{rm}\Url}\fi

\bibitem[Acuna et~al.(2021)Acuna, Zhang, Law, and Fidler]{acuna2021f}
David Acuna, Guojun Zhang, Marc~T Law, and Sanja Fidler.
\newblock f-domain adversarial learning: Theory and algorithms.
\newblock In \emph{International Conference on Machine Learning}, pages 66--75.
  PMLR, 2021.

\bibitem[Agrawal and Horel(2020)]{agrawal2020optimal}
Rohit Agrawal and Thibaut Horel.
\newblock Optimal bounds between f-divergences and integral probability
  metrics.
\newblock In \emph{International Conference on Machine Learning}, pages
  115--124. PMLR, 2020.

\bibitem[Ali and Silvey(1966)]{ali1966general}
Syed~Mumtaz Ali and Samuel~D Silvey.
\newblock A general class of coefficients of divergence of one distribution
  from another.
\newblock \emph{Journal of the Royal Statistical Society: Series B
  (Methodological)}, 28\penalty0 (1):\penalty0 131--142, 1966.

\bibitem[Amari(2007)]{amari2007integration}
Shun-ichi Amari.
\newblock Integration of stochastic models by minimizing $\alpha$-divergence.
\newblock \emph{Neural computation}, 19\penalty0 (10):\penalty0 2780--2796,
  2007.

\bibitem[Bartlett et~al.(2006)Bartlett, Jordan, and
  McAuliffe]{bartlett2006convexity}
Peter~L Bartlett, Michael~I Jordan, and Jon~D McAuliffe.
\newblock Convexity, classification, and risk bounds.
\newblock \emph{Journal of the American Statistical Association}, 101\penalty0
  (473):\penalty0 138--156, 2006.

\bibitem[Ben-David and Urner(2012)]{ben2012hardness}
Shai Ben-David and Ruth Urner.
\newblock On the hardness of domain adaptation and the utility of unlabeled
  target samples.
\newblock In \emph{International Conference on Algorithmic Learning Theory},
  pages 139--153. Springer, 2012.

\bibitem[Ben-David et~al.(2010)Ben-David, Blitzer, Crammer, Kulesza, Pereira,
  and Vaughan]{ben2010theory}
Shai Ben-David, John Blitzer, Koby Crammer, Alex Kulesza, Fernando Pereira, and
  Jennifer~Wortman Vaughan.
\newblock A theory of learning from different domains.
\newblock \emph{Machine learning}, 79\penalty0 (1-2):\penalty0 151--175, 2010.

\bibitem[Borji(2019)]{borji2019pros}
Ali Borji.
\newblock Pros and cons of gan evaluation measures.
\newblock \emph{Computer Vision and Image Understanding}, 179:\penalty0 41--65,
  2019.

\bibitem[Brock et~al.(2018)Brock, Donahue, and Simonyan]{brock2018large}
Andrew Brock, Jeff Donahue, and Karen Simonyan.
\newblock Large scale gan training for high fidelity natural image synthesis.
\newblock \emph{arXiv preprint arXiv:1809.11096}, 2018.

\bibitem[Csisz{\'a}r(1964)]{csiszar1964information-theoretic}
Imre Csisz{\'a}r.
\newblock An information-theoretic inequality and its application to proofs of
  the ergodicity of markoff chains.
\newblock \emph{Magyar Tud. Akad. Mat. Kutato Int. Koezl.}, 8:\penalty0
  85--108, 1964.

\bibitem[Davis and Goadrich(2006)]{davis2006relationship}
Jesse Davis and Mark Goadrich.
\newblock The relationship between precision-recall and roc curves.
\newblock In \emph{Proceedings of the 23rd international conference on Machine
  learning}, pages 233--240, 2006.

\bibitem[DeGroot(1962)]{degroot1962uncertainty}
Morris~H DeGroot.
\newblock Uncertainty, information, and sequential experiments.
\newblock \emph{The Annals of Mathematical Statistics}, 33\penalty0
  (2):\penalty0 404--419, 1962.

\bibitem[Djolonga et~al.(2019)Djolonga, Lucic, Cuturi, Bachem, Bousquet, and
  Gelly]{djolonga2019evaluating}
Josip Djolonga, Mario Lucic, Marco Cuturi, Olivier Bachem, Olivier Bousquet,
  and Sylvain Gelly.
\newblock Evaluating generative models using divergence frontiers.
\newblock \emph{arXiv preprint arXiv:1905.10768}, 2019.

\bibitem[Djolonga et~al.(2020)Djolonga, Lucic, Cuturi, Bachem, Bousquet, and
  Gelly]{pmlr-v108-djolonga20a}
Josip Djolonga, Mario Lucic, Marco Cuturi, Olivier Bachem, Olivier Bousquet,
  and Sylvain Gelly.
\newblock Precision-recall curves using information divergence frontiers.
\newblock In \emph{Proceedings of the Twenty Third International Conference on
  Artificial Intelligence and Statistics}, Proceedings of Machine Learning
  Research. PMLR, 2020.

\bibitem[Ganin et~al.(2016)Ganin, Ustinova, Ajakan, Germain, Larochelle,
  Laviolette, Marchand, and Lempitsky]{ganin2016domain}
Yaroslav Ganin, Evgeniya Ustinova, Hana Ajakan, Pascal Germain, Hugo
  Larochelle, Fran{\c{c}}ois Laviolette, Mario Marchand, and Victor Lempitsky.
\newblock Domain-adversarial training of neural networks.
\newblock \emph{The journal of machine learning research}, 17\penalty0
  (1):\penalty0 2096--2030, 2016.

\bibitem[Harremo{\"e}s(2004)]{harremoes2004new}
Peter Harremo{\"e}s.
\newblock A new look on majorization.
\newblock \emph{Proceedings ISITA 2004, Parma, Italy}, pages 1422--1425, 2004.

\bibitem[Heusel et~al.(2017)Heusel, Ramsauer, Unterthiner, Nessler, and
  Hochreiter]{heusel2017gans}
Martin Heusel, Hubert Ramsauer, Thomas Unterthiner, Bernhard Nessler, and Sepp
  Hochreiter.
\newblock Gans trained by a two time-scale update rule converge to a local nash
  equilibrium.
\newblock In \emph{Advances in Neural Information Processing Systems}, pages
  6626--6637, 2017.

\bibitem[Im et~al.(2018)Im, Ma, Taylor, and Branson]{im2018quantitatively}
Daniel~Jiwoong Im, He~Ma, Graham Taylor, and Kristin Branson.
\newblock Quantitatively evaluating gans with divergences proposed for
  training.
\newblock \emph{arXiv preprint arXiv:1803.01045}, 2018.

\bibitem[Johansson et~al.(2019)Johansson, Sontag, and
  Ranganath]{johansson2019support}
Fredrik~D Johansson, David Sontag, and Rajesh Ranganath.
\newblock Support and invertibility in domain-invariant representations.
\newblock In \emph{The 22nd International Conference on Artificial Intelligence
  and Statistics}, pages 527--536. PMLR, 2019.

\bibitem[Karras et~al.(2019)Karras, Laine, and Aila]{karras2019style}
Tero Karras, Samuli Laine, and Timo Aila.
\newblock A style-based generator architecture for generative adversarial
  networks.
\newblock In \emph{Proceedings of the IEEE Conference on Computer Vision and
  Pattern Recognition}, pages 4401--4410, 2019.

\bibitem[Kynk{\"a}{\"a}nniemi et~al.(2019)Kynk{\"a}{\"a}nniemi, Karras, Laine,
  Lehtinen, and Aila]{kynkaanniemi2019improved}
Tuomas Kynk{\"a}{\"a}nniemi, Tero Karras, Samuli Laine, Jaakko Lehtinen, and
  Timo Aila.
\newblock Improved precision and recall metric for assessing generative models.
\newblock In \emph{Advances in Neural Information Processing Systems}, pages
  3929--3938, 2019.

\bibitem[Liese and Vajda(2006)]{liese2006divergences}
Friedrich Liese and Igor Vajda.
\newblock On divergences and informations in statistics and information theory.
\newblock \emph{IEEE Transactions on Information Theory}, 52\penalty0
  (10):\penalty0 4394--4412, 2006.

\bibitem[Lin et~al.(2017)Lin, Khetan, Fanti, and Oh]{lin2017arxiv}
Zinan Lin, Ashish Khetan, Giulia~C. Fanti, and Sewoong Oh.
\newblock Pacgan: The power of two samples in generative adversarial networks.
\newblock \emph{CoRR}, abs/1712.04086, 2017.
\newblock URL \url{http://arxiv.org/abs/1712.04086}.

\bibitem[Lin et~al.(2018)Lin, Khetan, Fanti, and Oh]{lin2018pacgan}
Zinan Lin, Ashish Khetan, Giulia Fanti, and Sewoong Oh.
\newblock Pacgan: The power of two samples in generative adversarial networks.
\newblock In \emph{Advances in Neural Information Processing Systems}, pages
  1498--1507, 2018.

\bibitem[Liu et~al.(2021)Liu, Pillutla, Welleck, Oh, Choi, and
  Harchaoui]{liu2021divergence}
Lang Liu, Krishna Pillutla, Sean Welleck, Sewoong Oh, Yejin Choi, and Zaid
  Harchaoui.
\newblock Divergence frontiers for generative models: Sample complexity,
  quantization effects, and frontier integrals.
\newblock \emph{Advances in Neural Information Processing Systems}, 34, 2021.

\bibitem[Lorenz(1905)]{lorenz1905methods}
Max~O Lorenz.
\newblock Methods of measuring the concentration of wealth.
\newblock \emph{Publications of the American statistical association},
  9\penalty0 (70):\penalty0 209--219, 1905.

\bibitem[M{\"u}ller(1997)]{muller1997integral}
Alfred M{\"u}ller.
\newblock Integral probability metrics and their generating classes of
  functions.
\newblock \emph{Advances in Applied Probability}, 29\penalty0 (2):\penalty0
  429--443, 1997.

\bibitem[Naeem et~al.(2020)Naeem, Oh, Uh, Choi, and Yoo]{naeem2020reliable}
Muhammad~Ferjad Naeem, Seong~Joon Oh, Youngjung Uh, Yunjey Choi, and Jaejun
  Yoo.
\newblock Reliable fidelity and diversity metrics for generative models.
\newblock \emph{arXiv preprint arXiv:2002.09797}, 2020.

\bibitem[Nguyen et~al.(2010)Nguyen, Wainwright, and
  Jordan]{nguyen2010estimating}
XuanLong Nguyen, Martin~J Wainwright, and Michael~I Jordan.
\newblock Estimating divergence functionals and the likelihood ratio by convex
  risk minimization.
\newblock \emph{IEEE Transactions on Information Theory}, 56\penalty0
  (11):\penalty0 5847--5861, 2010.

\bibitem[Nowozin et~al.(2016)Nowozin, Cseke, and Tomioka]{nowozin2016f}
Sebastian Nowozin, Botond Cseke, and Ryota Tomioka.
\newblock f-gan: Training generative neural samplers using variational
  divergence minimization.
\newblock \emph{Advances in neural information processing systems}, 29, 2016.

\bibitem[Piccoli et~al.(2019)Piccoli, Rossi, and
  Tournus]{piccoli2019wasserstein}
Benedetto Piccoli, Francesco Rossi, and Magali Tournus.
\newblock A wasserstein norm for signed measures, with application to nonlocal
  transport equation with source term.
\newblock \emph{arXiv preprint arXiv:1910.05105}, 2019.

\bibitem[P{\'o}czos and Schneider(2011)]{poczos2011estimation}
Barnab{\'a}s P{\'o}czos and Jeff Schneider.
\newblock On the estimation of alpha-divergences.
\newblock In \emph{Proceedings of the Fourteenth International Conference on
  Artificial Intelligence and Statistics}, pages 609--617. JMLR Workshop and
  Conference Proceedings, 2011.

\bibitem[Redko et~al.(2020)Redko, Morvant, Habrard, Sebban, and
  Bennani]{redko2020survey}
Ievgen Redko, Emilie Morvant, Amaury Habrard, Marc Sebban, and Youn{\`e}s
  Bennani.
\newblock A survey on domain adaptation theory.
\newblock \emph{arXiv preprint arXiv:2004.11829}, 2020.

\bibitem[Rubenstein et~al.(2019)Rubenstein, Bousquet, Djolonga, Riquelme, and
  Tolstikhin]{rubenstein2019practical}
Paul Rubenstein, Olivier Bousquet, Josip Djolonga, Carlos Riquelme, and Ilya~O
  Tolstikhin.
\newblock Practical and consistent estimation of f-divergences.
\newblock In \emph{Advances in Neural Information Processing Systems}, pages
  4072--4082, 2019.

\bibitem[Ruderman et~al.(2012)Ruderman, Reid, Garc{\'\i}a-Garc{\'\i}a, and
  Petterson]{ruderman2012tighter}
Avraham Ruderman, Mark Reid, Dar{\'\i}o Garc{\'\i}a-Garc{\'\i}a, and James
  Petterson.
\newblock Tighter variational representations of f-divergences via restriction
  to probability measures.
\newblock \emph{arXiv preprint arXiv:1206.4664}, 2012.

\bibitem[Sajjadi et~al.(2018)Sajjadi, Bachem, Lucic, Bousquet, and
  Gelly]{sajjadi2018}
Mehdi S.~M. Sajjadi, Olivier Bachem, Mario Lucic, Olivier Bousquet, and Sylvain
  Gelly.
\newblock Assessing generative models via precision and recall.
\newblock In S.~Bengio, H.~Wallach, H.~Larochelle, K.~Grauman, N.~Cesa-Bianchi,
  and R.~Garnett, editors, \emph{Advances in Neural Information Processing
  Systems 31}, pages 5234--5243. Curran Associates, Inc., 2018.

\bibitem[Scott et~al.(2013)Scott, Blanchard, and
  Handy]{scott2013classification}
Clayton Scott, Gilles Blanchard, and Gregory Handy.
\newblock Classification with asymmetric label noise: Consistency and maximal
  denoising.
\newblock In \emph{Conference on learning theory}, pages 489--511. PMLR, 2013.

\bibitem[Simon et~al.(2019)Simon, Webster, and Rabin]{simon2019revisiting}
Loic Simon, Ryan Webster, and Julien Rabin.
\newblock Revisiting precision recall definition for generative modeling.
\newblock In Kamalika Chaudhuri and Ruslan Salakhutdinov, editors,
  \emph{Proceedings of the 36th International Conference on Machine Learning},
  volume~97 of \emph{Proceedings of Machine Learning Research}, pages
  5799--5808, Long Beach, California, USA, 09--15 Jun 2019. PMLR.

\bibitem[Siry et~al.(2021)Siry, H{\'e}madou, Simon, and
  Jurie]{siry2021inductive}
Rodrigue Siry, Louis H{\'e}madou, Lo{\"\i}c Simon, and Fr{\'e}d{\'e}ric Jurie.
\newblock On the inductive biases of deep domain adaptation.
\newblock \emph{arXiv preprint arXiv:2109.07920}, 2021.

\bibitem[Sriperumbudur et~al.(2009)Sriperumbudur, Fukumizu, Gretton,
  Sch{\"o}lkopf, and Lanckriet]{sriperumbudur2009integral}
Bharath~K Sriperumbudur, Kenji Fukumizu, Arthur Gretton, Bernhard
  Sch{\"o}lkopf, and Gert~RG Lanckriet.
\newblock On integral probability metrics,$\backslash$phi-divergences and
  binary classification.
\newblock \emph{arXiv preprint arXiv:0901.2698}, 2009.

\bibitem[Steinwart(2007)]{steinwart2007compare}
Ingo Steinwart.
\newblock How to compare different loss functions and their risks.
\newblock \emph{Constructive Approximation}, 26\penalty0 (2):\penalty0
  225--287, 2007.

\bibitem[Terj{\'e}k(2021)]{terjek2021moreau}
D{\'a}vid Terj{\'e}k.
\newblock Moreau-yosida $ f $-divergences.
\newblock In \emph{International Conference on Machine Learning}, pages
  10214--10224. PMLR, 2021.

\bibitem[Tewari and Bartlett(2007)]{tewari2007consistency}
Ambuj Tewari and Peter~L Bartlett.
\newblock On the consistency of multiclass classification methods.
\newblock \emph{Journal of Machine Learning Research}, 8\penalty0 (5), 2007.

\bibitem[Tsallis(1988)]{tsallis1988possible}
Constantino Tsallis.
\newblock Possible generalization of boltzmann-gibbs statistics.
\newblock \emph{Journal of statistical physics}, 52\penalty0 (1):\penalty0
  479--487, 1988.

\bibitem[van Erven and Harremo{\"e}s(2010)]{van2010renyi}
Tim van Erven and Peter Harremo{\"e}s.
\newblock R{\'e}nyi divergence and majorization.
\newblock In \emph{2010 IEEE International Symposium on Information Theory},
  pages 1335--1339. IEEE, 2010.

\bibitem[Wang et~al.(2019)Wang, Wang, Zhang, Owens, and Efros]{wang2019cnn}
Sheng-Yu Wang, Oliver Wang, Richard Zhang, Andrew Owens, and Alexei~A Efros.
\newblock Cnn-generated images are surprisingly easy to spot... for now.
\newblock \emph{arXiv preprint arXiv:1912.11035}, 2019.

\bibitem[Webster et~al.(2019)Webster, Rabin, Simon, and
  Jurie]{webster2019GenOverfit}
Ryan Webster, Julien Rabin, Loic Simon, and Fr\'ed\'eric Jurie.
\newblock Detecting overfitting of deep generative networks via latent
  recovery.
\newblock In \emph{Proceedings of the IEEE Conference on Computer Vision and
  Pattern Recognition (CVPR)}, 2019.

\bibitem[Zhang et~al.(2016)Zhang, Bengio, Hardt, Recht, and
  Vinyals]{zhang2016understanding}
Chiyuan Zhang, Samy Bengio, Moritz Hardt, Benjamin Recht, and Oriol Vinyals.
\newblock Understanding deep learning requires rethinking generalization.
\newblock \emph{arXiv preprint arXiv:1611.03530}, 2016.

\bibitem[Zhang et~al.(2019)Zhang, Liu, Long, and Jordan]{zhang2019bridging}
Yuchen Zhang, Tianle Liu, Mingsheng Long, and Michael Jordan.
\newblock Bridging theory and algorithm for domain adaptation.
\newblock In \emph{International Conference on Machine Learning}, pages
  7404--7413, 2019.

\bibitem[Zhao et~al.(2019)Zhao, Des~Combes, Zhang, and
  Gordon]{zhao2019learning}
Han Zhao, Remi~Tachet Des~Combes, Kun Zhang, and Geoffrey Gordon.
\newblock On learning invariant representations for domain adaptation.
\newblock In \emph{International Conference on Machine Learning}, pages
  7523--7532. PMLR, 2019.

\end{thebibliography}

\end{document}